\documentclass[11pt]{article} 
\usepackage{cdesa}

\usepackage[numbers,sort]{natbib}
\usepackage{multibib}

\usepackage{subfigure}

\newtheorem{mycondition}{Condition}

\newcites{sec}{Secondary Literature}

\usepackage{etoolbox}
\patchcmd{\thebibliography}{\section*{\refname}}{}{}{}

\title{Global Convergence of Stochastic Gradient Descent for Some Non-convex Matrix Problems}

\author{
Christopher De Sa,
Kunle Olukotun, and
Christopher R{\'e} \\
\texttt{cdesa@stanford.edu},
\texttt{kunle@stanford.edu},
\texttt{chrismre@stanford.edu} \\
Departments of Electrical Engineering and Computer Science\\
Stanford University, Stanford, CA 94309
}

\begin{document}

\maketitle

\begin{abstract} 
Stochastic gradient descent (SGD) on a low-rank factorization~\cite{burer2003}
is commonly employed to speed up matrix problems including matrix
completion, subspace
tracking, and SDP relaxation. 
In this paper, we exhibit a step size scheme for SGD on
a low-rank least-squares problem, and 
we prove that, under broad sampling
conditions, our method converges globally from a random
starting point within $O(\epsilon^{-1} n \log n)$ steps with constant
probability for constant-rank problems.
Our modification of SGD relates it to stochastic power
iteration.  We also show experiments to illustrate the runtime
and convergence of the algorithm.
\end{abstract}

\section{Introduction}
We analyze an algorithm to solve the stochastic
optimization problem
\begin{equation}
  \label{eqnFullProblem}
  \begin{array}{ll}
    \mbox{minimize} & \Exv{\normf{\tilde A - X}^2} \\
    \mbox{subject to} & X \in \R^{n \times n}, \rank{X} \le p, X \succeq 0,
  \end{array}
\end{equation}
where $p$ is an integer and $\tilde A$ is a symmetric matrix drawn from
some distribution with bounded covariance.
The solution to this problem is the matrix formed by
zeroing out all but the largest $p$ eigenvalues of the matrix $\mathbf{E}[\tilde A]$. 
This problem, or problems that can be transformed to this problem, appears
in a variety of machine learning
applications including matrix completion~\cite{jain2012,teflioudi2013,chen2011},
general data analysis~\cite{zou2004}, subspace tracking~\cite{grouse},
principle component analysis~\cite{Arora2012},
optimization~\cite{burer2005,journee,mishra2013low,Tropp2014},
and recommendation systems~\cite{Gupta:2013,Boykin2014}.

Sometimes, (\ref{eqnFullProblem}) arises under conditions in which
the samples $\tilde A$ are sparse, but the matrix $X$ would be too large
to store and operate on efficiently; a standard heuristic to use in this
case is a low-rank factorization~\cite{burer2003}.  
The idea is to substitute $X = Y Y^T$ and solve the
problem
\begin{equation}
  \label{eqnBMProblem}
  \begin{array}{ll}
    \mbox{minimize} & \Exv{\normf{\tilde A - Y Y^T}^2} \\
    \mbox{subject to} & Y \in \R^{n \times p}.
  \end{array}
\end{equation}
By construction, if we set $X = Y Y^T$, then $X \in \R^{n \times n}$,
$\rank{X} \le p$, and $X \succeq 0$; this allows us
to drop these constraints.  Instead
of having to store the matrix $X$ (of size $n^2$), we only need
to store the matrix $Y$ (of size $np$).

In practice, many people use stochastic
gradient descent (SGD) to solve (\ref{eqnBMProblem}).
Efficient SGD implementations can scale
to very large
datasets~\cite{jellyfish,hogwild,teflioudi2013,VowpalWabbit,Bottou,AdaGrad,Bottou2008,Hu09acceleratedgradient}.
However, standard stochastic gradient descent on (\ref{eqnBMProblem})
does not converge globally, in the sense that there will always
be some initial values for which the norm of the iterate will diverge
(see Appendix \ref{ssCounterexamples}).

People have attempted to compensate for this with
sophisticated methods like geodesic step rules~\cite{journee} and manifold
projections~\cite{absil}; however, even these methods cannot
guarantee global convergence.
Motivated by this, we describe Alecton, an algorithm for solving
(\ref{eqnBMProblem}), and analyze its convergence.  Alecton is
an SGD-like algorithm that has a simple update rule with a step
size that is a simple function of the norm of the iterate $Y_k$.
We show that Alecton converges globally.
We make the following contributions:
\begin{itemize}
  \item We establish the convergence rate to a global optimum 
    of Alecton using a random initialization;
    in contrast, prior analyses~\cite{candes2014phase,jain2012} have required more expensive
    initialization methods, such as the singular value decomposition
    of an empirical average of the data.
  \item In contrast to previous work that uses bounds on the magnitude
    of the noise~\cite{Hardt2014}, our analysis depends only on the variance of
    the samples.  As a result,
    we are able to be robust to different noise models, and we apply
    our technique to these problems, which did not previously
    have global convergence rates:
    \begin{itemize}
    \item \emph{matrix completion}, in which we observe entries
      of $A$ one at a time~\cite{jain2012,keshavan2010} (Section \ref{sec:matrix:c}),
    \item \emph{phase retrieval}, in which we observe
      $\mathrm{tr}(u^TAv)$ for randomly selected $u,v$~\cite{candes2014phase,phaselift2014} (Section~\ref{sec:phase}), and
    \item \emph{subspace tracking}, in which $A$ is a projection matrix
      and we observe random entries of a random vector in its column space~\cite{grouse}
      (Section~\ref{sec:subspace}).
    \end{itemize}
    Our result is also robust to different noise models.
  \item We describe a martingale-based analysis technique that is novel
    in the space of non-convex optimization.  We are able to generalize
    this technique to some simple regularized problems, and we are optimistic
    that it has more applications.
\end{itemize}

\subsection{Related Work}

Much related work exists in the space of solving low-rank factorized
optimization problems.  Foundational work in this space was done by
Burer and Monteiro~\cite{burer2003,burer2005}, who analyzed the
low-rank factorization of general semidefinite programs.
Their results focus on the classification of the local minima of such
problems, and on conditions under which no non-global minima exist.
They do not analyze the convergence rate of SGD.

Another general analysis in \citet{journee} exhibits a second-order
algorithm that converges to a local solution.  Their results
use manifold optimization
techniques to optimize over the manifold of low-rank matrices.
These approaches have attempted
to correct for falling off the manifold using Riemannian
retractions~\cite{journee},
geodesic steps~\cite{grouse}, or projections back onto the manifold.
General non-convex manifold optimization techniques~\cite{absil}
tell us that first-order
methods, such as SGD, will converge to a fixed point, but they
provide no convergence rate to the global optimum.
Our algorithm only involves a simple rescaling, and we are able
to provide global convergence results.

Our work follows others who have studied individual problems
that we consider.
~\citet{jain2012} study matrix completion and provides a convergence
rate for an exact recovery algorithm, alternating minimization.
\citet{candes2014phase}
provide a similar result for phase retrieval.
In contrast to these results, which require expensive 
SVD-like operations to initialize, our results allow random
initialization.  Our provided convergence
rates apply to additional problems and SGD
algorithms that are used in practice (but are not covered by
previous analysis).  However, our convergence rates are slower
in their respective settings.  This is likely unavoidable
in our setting, as we show that our convergence rate is optimal
in this more general setting (see Appendix \ref{ssLowerBoundSGD}).

A related class of algorithms that are similar to Alecton is
stochastic power iteration~\cite{Arora2012}.
These algorithms reconsider (\ref{eqnFullProblem}) as an
eigenvalue problem, and uses the familiar power iteration
algorithm, adapted to a stochastic setting.
Stochastic power iteration has been applied to a wide variety of
problems~\cite{Arora2012,Goes2014}.
\citet{Oja1985} show convergence of this algorithm,
but provides no rate.
\citet{Arora2013} analyze this problem, and state that
``obtaining a theoretical understanding of the stochastic power
method, or of how the step size should be set, has proved
elusive.''  Our paper addresses this by providing a method for
selecting the step size, although our analysis shows convergence
for any sufficiently small step size.

\citet{Shamir2014} provide exponential-rate local convergence
results for a stochastic power iteration algorithm for PCA.
As they note, it can be used in practice to improve the accuracy
of an estimate returned by another, globally-convergent algorithm
such as Alecton.

Also recently, \citet{Balsubramani2013} and \citet{Hardt2014}
provide a global convergence rate for the stochastic power
iteration algorithm.
Our result only
depends on the variance of the samples, while both their results
require absolute bounds on the magnitude of the noise.  This allows
us to analyze a different class of noise models, which enables us
to do matrix completion, phase retrieval, and subspace tracking
in the same model.

\section{Algorithmic Derivation}

We focus on the low-rank factorized stochastic optimization
problem (\ref{eqnBMProblem}).
We can rewrite the objective as $\Exv{\tilde f(Y)}$,
with sampled objective function
\[
  \tilde f(Y)
  =
  \trace{Y Y^T Y Y^T}
  -
  2 \trace{Y \tilde A Y^T}
  +
  \normf{\tilde A}^2.
\]
In the analysis that follows, we let $A = \Exv{\tilde A}$, and
let its eigenvalues be
$\lambda_1 \ge \lambda_2 \ge \cdots \ge \lambda_n$ with
corresponding orthonormal eigenvectors $u_1, u_2, \ldots, u_n$
(such a decomposition is guaranteed since $A$ is symmetric).
The standard stochastic gradient descent update rule for this problem
is, for some step size $\alpha_k$,
\begin{align*}
  Y_{k+1}
  &=
  Y_k - \alpha_k \nabla \tilde f_k(Y) \\
  &=
  Y_k
  -
  4 \alpha_k \left(Y_k Y_k^T Y_k - \tilde A_k Y_k \right),
\end{align*}
where $\tilde A_k$ is the sample we use at timestep $k$.

The low-rank factorization introduces symmetry into the problem.
If we let
\[
  \mathcal{O}_p = \left\{ U \in \R^{p \times p} \mid U^T U = I_p \right\}
\]
denote the set of orthogonal matrices in $\R^{p \times p}$, then
$\tilde f(Y) = \tilde f(Y U)$ for any $U \in \mathcal{O}_p$.
Previous work has used manifold
optimization techniques to solve such symmetric
problems~\cite{journee}.  \citet{absil} state that stochastic
gradient descent on a manifold has the general form
\[
  x_{k+1}
  =
  x_k - \alpha_k G^{-1}_{x_k} \nabla \tilde f_k(x_k), 
\]
where $G_x$ is the matrix such that for all $u$ and $v$,
\[
  u^T G_x v = \langle u, v \rangle_x,
\]
where the right side of this equation denotes the
\emph{Riemannian metric}~\cite{do1992riemannian}
of the manifold at $x$.  For (\ref{eqnBMProblem}), the manifold in
question is
\[
  \mathcal{M} = \R^{n \times p} / \mathcal{O}_p,
\]
which is the quotient manifold of $\R^{n \times p}$ under the
orthogonal group action.  According to \citet{absil}, this
manifold has induced Riemannian metric
\begin{equation}
  \label{eqnAbsilRMetric}
  \langle U, V \rangle_Y = \trace{U Y^T Y V^T}.
\end{equation}
For Alecton, we are free to pick any Riemannian metric and step
size.  Inspired by (\ref{eqnAbsilRMetric}), we pick a new step
size parameter $\eta$, and let $\alpha_k = \frac{1}{4} \eta$ and set
\[
  \langle U, V \rangle_Y = \trace{U (I + \eta Y^T Y) V^T}.
\]
With this, the SGD update rule becomes
\begin{align*}
  Y_{k+1}
  &=
  Y_k
  -
  \eta \left(Y_k Y_k^T Y_k - \tilde A_k Y_k \right) \left(I + \eta Y_k^T Y_k \right)^{-1} \\
  &=
  \left(
    Y_k \left(I + \eta Y_k^T Y_k \right)
    -
    \eta \left(Y_k Y_k^T Y_k - \tilde A_k Y_k \right)
  \right)
  \left(I + \eta Y_k^T Y_k \right)^{-1} \\
  &=
  \left(I + \eta \tilde A_k\right) Y_k
  \left(I + \eta Y_k^T Y_k \right)^{-1}.
\end{align*}
For $p = 1$, choosing a Riemannian metric to use with SGD results
in the same algorithm as choosing an SGD step size that depends on the
iterate $Y_k$.  The same update rule would result if we substituted
\[
  \alpha_k = \frac{1}{4} \eta \left(1 + \eta Y^T Y \right)^{-1}
\]
into the standard SGD update formula.  We can think of this as the manifold
results giving us intuition on how to set our step size.

The reason why selecting this particular step size/metric is useful in
practice is that we can run the simpler update rule
\begin{equation}
  \label{eqPowerIterUpdate}
  \bar Y_{k+1} = \left(I + \eta \tilde A_k\right) \bar Y_k.
\end{equation}
If $\bar Y_0 = Y_0$,
the iteration will satisfy the property that the column space
of $Y_k$ will always be equal to the column space of $\bar Y_k$,
(since $C(XY) = C(X)$ for any invertible matrix $Y$).
That is, if we just care about computing the column space of $Y_k$,
we can do it using the much simpler update rule (\ref{eqPowerIterUpdate}).
Intuitively, we have transformed an optimization
problem operating in the whole space $\R^n$ to one operating on the
Grassmannian; one benefit of Alecton is that we don't have to work
on the actual Grassmannian, but get some of the same benefits from a
rescaling of the $Y_k$ space.
In this specific case, the Alecton update rule is akin to stochastic power
iteration, since it involves a repeated multiplication by the sample;
this would not hold for optimization on other manifolds.

We can use (\ref{eqPowerIterUpdate}) to compute the
column space (or ``angular component'') of the solution, before then recovering
the rest of the solution (the ``radial component'') using averaging.
Doing this corresponds to Algorithm \ref{algAlecton}, Alecton.  
Notice that, unlike most iterative algorithms for matrix recovery, Alecton
does not require any special initialization phase and can be initialized randomly.

\begin{algorithm}[h]
  \caption{Alecton: Solve stochastic matrix problem}
  \begin{algorithmic}
  \label{algAlecton}
    \REQUIRE $\eta \in \R$, $K \in \N$, $L \in \N$, and a sampling
      distribution $\mathcal{A}$
    \STATE \(\triangleright\) \textbf{\textit{Angular component (eigenvector) estimation phase}}
    \STATE Select $Y_0$ uniformly in $\R^{n \times m}$ s.t. $Y_0^T Y_0 = I$.
    \FOR{$k = 0$ \TO $K-1$}
      \STATE Select $\tilde A_k$ uniformly and independently at random from
        the sampling distribution $\mathcal{A}$.
      \STATE $Y_{k+1} \leftarrow Y_k + \eta \tilde A_k Y_k$
    \ENDFOR
    \STATE $\hat Y \leftarrow Y_K \left(Y_K^T Y_K\right)^{-\frac{1}{2}}$
    \STATE \(\triangleright\) \textbf{\textit{Radial component (eigenvalue) estimation phase}}
    \STATE $R_0 \leftarrow 0$
    \FOR{$l = 0$ \TO $L-1$}
      \STATE Select $\tilde A_l$ uniformly and independently at random from
        the sampling distribution $\mathcal{A}$.
      \STATE $R_{l+1} \leftarrow R_l + \hat Y^T \tilde A_l \hat Y$
    \ENDFOR
    \STATE $\bar R \leftarrow R_L / L$
    \RETURN $\hat y \bar R^{\frac{1}{2}}$
  \end{algorithmic}
\end{algorithm}

\paragraph{Analysis}
Analyzing this algorithm is challenging, as
the low-rank decomposition also introduces symmetrical families of
fixed points.  Not all these points are globally optimal: in fact,
a fixed point will occur whenever
\[
  Y Y^T = \sum_{i \in C} \lambda_i u_i u_i^T
\]
for any set $C$ of size less than $p$.

One consequence of the non-optimal fixed points is that
the standard proof of SGD's convergence, in which we choose a
Lyapunov function and show that this function's expectation decreases
with time, cannot work.  This is because, if such a Lyapunov
function were to exist, it would show that no matter where we
initialize the iteration, convergence to a global optimum
will still occur rapidly; this
cannot be possible due to the presence of the non-optimal fixed points.
Thus, a standard statement of global convergence,
that convergence occurs uniformly regardless of initial condition, 
cannot hold.

We therefore use martingale-based methods to show convergence.
Specifically, our attack involves defining a process $x_k$ with
respect to the natural filtration $\F_k$ of the iteration, such that
$x_k$ is a supermartingale, that is $\Exvc{x_{k+1}}{\F_k} \le x_k$. We
then use the \emph{optional stopping theorem}~\cite{fleming1991}
to bound both the
probability and rate of convergence of $x_k$, from which we derive
convergence of the original algorithm.  We describe this analysis
in the next section.

\section{Convergence Analysis}
First, we need a way to define convergence for the angular phase.
For most problems, we want $C(Y_k)$ to be as close as possible to the
span of $u_1,u_2,\ldots,u_p$.  However, for some cases, this is not what
we want.  For example, consider the case where $p = 1$ 
but $\lambda_1 = \lambda_2$.  In this case, the algorithm could not
recover $u_1$, since it is indistinguishable
from $u_2$.  Instead, it is reasonable to expect $C(Y_k)$ to converge
to the span of $u_1$ and $u_2$.

To handle this case, we instead want to measure convergence to
the subspace spanned by some number, $q \ge p$, of the algebraically
largest
eigenvectors (in most cases, $q = p$).  For a particular $q$, let $U$ be
the projection matrix onto the subspace spanned by $u_1, u_2, \ldots, u_q$,
and define $\Delta$, the \emph{eigengap}, as $\Delta = \lambda_q - \lambda_{q+1}$.
We now let $\epsilon > 0$ be an arbitrary error term, and
define an angular success condition for Alecton.
\begin{definition}
  \label{defnAlectonSuccess}
  When running the angular phase of Alecton, we say that
  \emph{success has occurred} at timestep $k$ if and only if for all
  $z \in \R^p$,
  \[
    \frac{\norm{U Y_k z}^2}{\norm{Y_k z}^2} \ge 1 - \epsilon.
  \]
  This condition requires that all members of the column
  space of $Y_k$ are close to the desired subspace.
  We say that \emph{success has occurred by time} $t$ if success has
  occurred for some timestep $k < t$.  Otherwise, we say the algorithm
  has \emph{failed}, and we let $F_t$ denote this failure event.
\end{definition}

To prove convergence, we need to put some restrictions on the problem.
Our theorem requires the following three conditions.
\begin{mycondition}[Alecton Variance]
  A sampling distribution $\mathcal{A}$ with expected value
  $A$ satisfies the \emph{Alecton Variance Condition} (AVC) with parameters
  $(\sigma_a, \sigma_r)$ if and only if for any $y \in \R$ and for any symmetric
  matrix $W \succeq 0$ that commutes with $A$,
  if $\tilde A$ is sampled from $\mathcal{A}$, the following bounds hold:
  \[
    \Exv{y^T \tilde A^T W \tilde A y}
    \le
    \sigma_a^2 \trace{W} \norm{y}^2
  \]
  and
  \[
    \Exv{\left(y^T \tilde A y\right)^2}
    \le
    \sigma_r^2 \norm{y}^4.
  \]
\end{mycondition}
In Section \ref{secExamples}, we show several models that satisfy AVC.
\begin{mycondition}[Alecton Rank]
  An instance of Alecton satisfies the \emph{Alecton Rank Condition} if
  either $p = 1$ (rank-$1$ recovery), or each sample $\tilde A$ from $\mathcal{A}$
  is rank-$1$ (rank-$1$ sampling).
\end{mycondition}
Most of the noise models we analyze have rank-1 samples, and so satisfy the rank
condition.
\begin{mycondition}[Alecton Step Size]
  Define $\gamma$ as
  \[
    \gamma
    =
    \frac{
      2 n \sigma_a^2 p^2 (p + \epsilon)
    }{
      \Delta \epsilon
    }
    \eta.
  \]
  This represents a constant step size parameter that is independent of problem
  scaling.  An instance of Alecton satisfies the 
  \emph{Alecton Step Size Condition} if and only if $\gamma \le 1$.
\end{mycondition}
Note that the step size condition is only an upper bound on the step size.
This means that, even if we do not know the problem parameters exactly, we can still choose a
feasible step size as long as we can bound them.  (However, smaller step sizes
imply slower convergence, so it is a good idea to choose $\eta$ as large as possible.)

We will now define a useful function, then 
state our main theorem that bounds the probability of failure.
\begin{definition}
  \label{defnSFxn}
  For some $p$, let $R \in \R^{p \times p}$ be a random matrix the entries of which
  are independent standard normal random variables.  Define function $Z_p$ as
  \[
    Z_p(\gamma) = 2 \left(1 - \Exv{\Det{I + \gamma p^{-1} (R^T R)^{-1}}^{-1}} \right).
  \]
\end{definition}

\begin{theorem}
  \label{thmAlecton}
  Assume that we run an instance of Alecton that satisfies the variance,
  rank, and step size conditions.  Then for any $t$, the probability
  that the angular phase will have failed up to time $t$ is
  \begin{equation}
    \label{eqnFailureBound}
    \Prob{F_t}
    \le
    Z_p(\gamma)
    +
    \frac{4 n \sigma_a^2 p^2 (p + \epsilon)}{\Delta^2 \gamma \epsilon t}
    \log\left( \frac{n p^2}{\gamma q \epsilon} \right).
  \end{equation}
  Also, in the radial phase, for any constant $\psi$ it holds that
  \[
    \Prob{
      \normf{\bar R - \hat Y^T A \hat Y}^2
      \ge 
      \psi
    }
    \le
    \frac{p^2 \sigma_r^2}{L \psi}.
  \]
\end{theorem}
In particular, if $\sigma_a \Delta^{-1}$ does not vary with $n$, this
theorem implies convergence of the angular phase
with constant probability after $O(\epsilon^{-1} n p^3 \log n)$
iterations and in the same amount of time.
Note that since we do not reuse samples in Alecton, our rates do not
differentiate between sampling and computational complexity, unlike many other
algorithms (see Appendix \ref{ssComparison}).  We also do not consider
numerical error or overflow: periodically re-normalizing the iterate may
be necessary to prevent these in an implementation of Alecton.

Since the upper bound expression uses $Z_p$, which is
obscure, we plot it here (Figure \ref{figSBound}).
We also can make a more precise statement about the failure rate
for $p = 1$.
\begin{lemma}
\label{lemmaZ1}
For the case of rank-$1$ recovery,
  \[
    Z_1(\gamma)
    =
    \sqrt{2 \pi \gamma}
    \exp\left(\frac{\gamma}{2} \right)
    \mathrm{erfc}\left(\sqrt{\frac{\gamma}{2}} \right)
    \le
    \sqrt{2 \pi \gamma}.
  \]
\end{lemma}

\begin{figure}[h]%
\centering
\resizebox{!}{.30\textwidth}{%
\Large \input{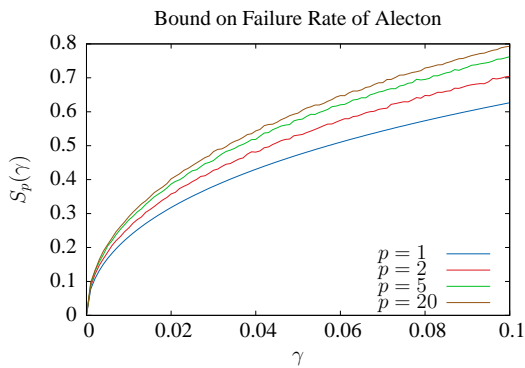}%
}%
\caption{Value of $Z_p$ computed as average of $10^5$ samples.}%
\label{figSBound}%
\end{figure}



\subsection{Martingale Technique}
\label{ssMartingaleTechnique}
A proof for Theorem \ref{thmAlecton} and full formal definitions
will appear in Appendix \ref{ssaProofs} of this document,
but since the
method is nonstandard for non-convex optimization
(although it has been used in \citet{Shamir2011} to show convergence for convex problems),
we will outline it here.  
First, we define a
\emph{failure event} $f_k$ at each timestep,
that occurs if the iterate gets ``too close''
to the unstable fixed points.
Next, we define a sequence $\tau_k$, where
\[
  \tau_k
  =
  \frac{
    \Det{Y_k^T U Y_k}
  }{
    \Det{Y_k^T \left(
      \gamma n^{-1} p^{-2} q I 
      +
      (1 - \gamma n^{-1} p^{-2} q) U
    \right) Y_k}
  }
\]
(where $\Det{X}$ denotes the determinant of $X$);
the intuition here is that $\tau_k$ is close to $1$ if and only
if success occurs, and close to $0$ when failure occurs.
We show that, if neither success nor failure occurs at time $k$,
\begin{equation}
  \label{taukeqn}
  \Exvc{\tau_{k+1}}{\F_k} \ge \tau_k \left(1 + R \left(1 - \tau_k \right) \right)
\end{equation}
for some constant $R$; here, $\F_k$ denotes the \emph{filtration} at time $k$, 
which contains all the events that have occurred up to time
$k$~\cite{fleming1991}.
If we let $T$ denote the first time at which either success or failure occurs,
then this implies that $\tau_k$ is a submartingale for $k < T$.  
We use the optional stopping Theorem~\cite{fleming1991} (here we state
a discrete-time version).
\begin{definition}[Stopping Time]
  A random variable $T$ is a stopping time with respect to a filtration $\F_k$
  if and only if $\left\{T \le k \right\} \in \F_k$ for all $k$.  That is, we can tell whether
  $T \le k$ using only events that have occurred up to time $k$.  
\end{definition}
\begin{theorem}[Optional Stopping Theorem]
  If $x_k$ is a martingale (or submartingale) with respect to a filtration
  $\F_k$, and $T$ is a stopping time with respect to the same filtration, then
  $x_{k \wedge T}$ is also a martingale (resp. submartingale) with respect to
  the same filtration, where $k \wedge T$ denotes the minimum of $k$ and $T$.
  In particular, for bounded submartingales, this implies that
  $\Exv{x_0} \le \Exv{x_T}$.
\end{theorem}
Here, $T$ is a stopping
time since it depends only on events occurring before timestep $T$.
Applying this to the submartingale $\tau_k$ results in
\begin{align*}
  \Exv{\tau_0}
  & \le
  \Exv{\tau_T} \\
  &=
  \Exvc{\tau_T}{F_T} \Prob{f_T} + \Exvc{\tau_T}{\lnot F_T} (1 - \Prob{f_T}) \\
  & \le
  \delta \Prob{f_T} + (1 - \Prob{f_T}).
\end{align*}
This isolates the probability of the failure event occurring.
Next, subtracting $1$ from
both sides of (\ref{taukeqn}) and taking the logarithm results in
\begin{align*}
  \Exvc{\log\left(1 - \tau_{k+1}\right)}{\F_k}
  & \le 
    \log(1 - \tau_k) + \log\left(1 - R \tau_k \right) \\
  & \le
    \log(1 - \tau_k) - R \delta.
\end{align*}
So, if we let $W_k = \log(1 - \tau_k) + R \delta k$, then $W_k$ is a supermartingale.
We again apply the optional stopping theorem to produce
\[
  \Exv{W_0}
  \ge
  \Exv{W_T}
  =
  \Exv{\log(1 - \tau_T)} + R \delta \Exv{T}.
\]
This isolates the expected value of the stopping time.  Finally, we notice
that success occurs before time $t$ if $T \le t$ and $f_T$ does not occur.
By the union bound, this implies that
\[
  P_{\text{failure}}
  \le
  \Prob{f_T} + \Prob{T \le t},
\]
and by Markov's inequality,
\[
  P_{\text{failure}}
  \le
  \Prob{f_T} + t^{-1} \Exv{T}.
\]
Substituting the isolated values for $\Prob{f_T}$ and $\Exv{T}$ produces
the expression above in (\ref{eqnFailureBound}).

The radial part of the theorem
follows from an application of Chebychev's inequality to the average of
$L$ samples of $\hat y^T \tilde A \hat y$ --- we do not devote any discussion
to it since averages are already well understood.

\section{Application Examples}
\label{secExamples}

\subsection{Entrywise Sampling}
\label{sec:matrix:c}
One sampling distribution that arises in many applications (most importantly,
matrix completion~\cite{candes2008}) is
\emph{entrywise sampling}.  This occurs when the samples are independently
chosen from the entries of $A$.  Specifically,
\[
  \tilde A = n^2 e_i e_i^T A e_j e_j^T,
\]
where $i$ and $j$ are each independently drawn from ${1,\ldots,n}$.
It is standard for these types of problems to introduce
a \emph{matrix coherence bound}~\cite{jain2012}.
\begin{definition}
  A matrix $A \in \R^{n \times n}$ is incoherent with parameter $\mu$ if and
  only if for every unit eigenvector $u_i$ of the matrix,
  and for all standard basis vectors $e_j$,
  \[
    \Abs{e_j^T u_i} \le \mu n^{-\frac{1}{2}}.
  \]
\end{definition}
Under an incoherence assumption, we can provide a bound on the second
moment of $\tilde A$, which is all that we need to apply Theorem
\ref{thmAlecton} to this problem.
\begin{lemma}
  \label{lemmaEntrywiseAVC}
  If $A$ is incoherent with parameter $\mu$, and $\tilde A$ is sampled uniformly from the
  entries of $A$, then the distribution of $\tilde A$ satisfies the Alecton variance
  condition with parameters
  $\sigma_a^2 = \mu^4 \normf{A}^2$ and $\sigma_r^2 = \mu^4 \trace{A}^2$.
\end{lemma}
For problems in which the matrix $A$ is of constant rank, and its eigenvalues do not vary
with $n$, neither $\normf{A}$ nor $\trace{A}$ will vary with $n$.  In this case,
$\sigma_a^2$, $\sigma_r^2$, and $\Delta$ will be constants, and the
$O(\epsilon^{-1} n \log n)$ bound on convergence time will hold.



\subsection{Rectangular Entrywise Sampling}
\label{sec:matrix:rectc}
Entrywise sampling also commonly appear in rectangular matrix recovery problems.  In these
cases, we are trying to solve something like 
\[
  \begin{array}{ll}
    \mbox{minimize} & \normf{M - X}^2 \\
    \mbox{subject to} & X \in \R^{m \times n}, \rank{X} \le p.
  \end{array}
\]
To solve this problem using Alecton, we first convert it into a symmetric
matrix problem by constructing the block matrix
\[
  A = \left[ \begin{array}{cc} 0 & M \\ M^T & 0 \end{array} \right];
\]
it is known that recovering the dominant eigenvectors of $A$
is equivalent to recovering the dominant singular vectors of $M$.

Entrywise sampling on $M$ corresponds to choosing a random $i \in {1,\ldots,m}$
and $j \in {1,\ldots,n}$,
and then sampling $\tilde A$ as
\[
  \tilde A = m n M_{ij} (e_i e_{m+j}^T + e_{m+j} e_i^T).
\]
In the case where we can bound the entries of $M$ (this is natural for
recommender systems), we can prove the following.
\begin{lemma}
  \label{lemmaEntrywiseRectAVC}
  If $M \in \R^{m \times n}$ satisfies the entry bound
  \[
    M_{ij}^2 \le \xi m^{-1} n^{-1} \normf{M}^2
  \]
  for all $i$ and $j$, then the rectangular entrywise sampling distribution on $M$
  satisfies the Alecton variance condition with parameters
  \[
    \sigma_a^2 = \sigma_r^2 = 2 \xi \normf{M}^2.
  \]
\end{lemma}
As above, for problems in which the singular values of $M$ do not vary with problem size,
our big-$O$ convergence time bound will still hold.


\subsection{Trace Sampling}
\label{sec:phase}
Another common sampling distribution arises from the \emph{matrix sensing}
problem~\cite{jain2012}. In this problem, we are given the value of $v^T A w$
for unit vectors $v$ and $w$ selected uniformly at random.  (This problem
has been handled for the more general complex case in \cite{candes2014phase}
using Wirtinger flow.)
Using a trace sample, we can construct an unbiased sample
\[
  \tilde A = n^2 v v^T A w w^T.
\]
This lets us bound the variance as follows.
\begin{lemma}
  \label{lemmaTraceAVC}
  If $n > 50$, and $v$ and $w$ are sampled uniformly from the unit sphere
  in $\R^n$, then for any positive semidefinite matrix $A$, if we let
  $\tilde A = n^2 v v^T A w w^T$, then
  the distribution of $\tilde A$ satisfies the Alecton variance
  condition with parameters
  $\sigma_a^2 = 16 \normf{A}^2$ and $\sigma_r^2 = 16 \trace{A}^2$.
\end{lemma}
As above, for problems in which the eigenvalues of $A$ do not vary with problem size,
our big-$O$ convergence time bound will still hold.



In some cases of the trace sampling problem, instead of being given samples
of the form $u^T A v$,
we know $u^T A u$.  In this case, we need to use two independent samples
$u_1^T A u_1$ and
$u_2^T A u_2$, and let $u \propto u_1 + u_2$ and $v \propto u_1 - u_2$
be two unit vectors which
we will use in the above sampling scheme.  Notice that since $u_1$
and $u_2$ are independent and
uniformly distributed, $u$ and $v$ will also be independent and uniformly
distributed (by the spherical
symmetry of the underlying distribution).  Furthermore, we can compute
\[
  u^T A v = (u_1 + u_2)^T A (u_1 - u_2) = u_1^T A u_1 - u_2^T A u_2.
\]
This allows us to use our above trace sampling scheme even with samples
of the form $u^T A u$.

\subsection{Subspace Sampling}
\label{sec:subspace}
Our analysis can handle more complicated sampling schemes.  Consider the following
distribution, which arises in subspace tracking~\cite{grouse}.  Our matrix $A$ is a rank-$r$
projection matrix, and each sample consists of some randomly-selected entries from a
randomly-selected vector in its column space.  Specifically, we are given $Qv$ and $Rv$, where
$v$ is some vector selected uniformly at random from $C(A)$, and $Q$ and $R$ are independent
random diagonal projection matrices with expected value $m n^{-1} I$.  Using this,
we can construct the distribution
\[
  \tilde A = r n^2 m^{-2} Q v v^T R.
\]
This distribution is unbiased
since $\Exv{q v v^T} = A$.  When bounding its second moment, we run into the same coherence
problem as we did in the entrywise case, which motivates us to introduce a coherence
constraint for subspaces.
\begin{definition}
  A subspace of $\R^n$ of dimension $q$ with associated projection matrix $U$ is incoherent
  with parameter $\mu$ if and only if for all standard basis vectors $e_i$,
  \[
    \norm{U e_i}^2 \le \mu r n^{-1}.
  \]
\end{definition}
Using this, we can prove the following facts about the second moment of this distribution.
\begin{lemma}
  \label{lemmaSubspaceAVC}
  The subspace sampling distribution, when sampled from a subspace that is incoherent with
  parameter $\mu$, satisfies the Alecton variance condition with parameters
  \[
    \sigma_a^2 = \sigma_r^2
    =
    r^2 (1 + \mu r m^{-1})^2.
  \]
\end{lemma}

In many cases of subspace sampling, we are given just some entries of $v$ at each timestep
(as opposed to two separate random sets of entries associated with $Q$ and $R$).  That is,
we are given a random diagonal projection matrix $S$, and the product $Sv$.  We can use
this to construct a sample of the above form by randomly splitting the given entries among
$Q$ and $R$ in such a way that $Q = QS$ and $R = RS$, and $Q$ and $R$ are independent.
We can then construct an unbiased sample as
\[
  \tilde A = r n^2 m^{-2} Q S v v^T S R,
\]
which uses only the entries of $v$ that we are given.

\subsection{Noisy Sampling}
\label{sec:noisy}
Since our analysis depends only on a variance bound, it is straightforward to handle
the case
in which the values of our samples themselves are noisy.  Using the additive
property of the variance for independent random variables,
we can show that additive noise only increases the
variance of the sampling distribution by a constant amount proportional to
the variance of the noise.
Similarly, using the multiplicative property of the variance for independent
random variables, multiplicative noise only multiplies
the variance of the sampling distribution by a constant factor proportional
to the variance of the noise.  In either case, we can show that
the noisy sampling distribution satisfies AVC.

\subsection{Extension to Higher Ranks}
\label{ssHigherRank}
It is possible to use multiple iterations of the rank-$1$ version of Alecton to
recover additional
eigenvalue/eigenvector pairs of the data matrix $A$ one-at-a-time.
This is a standard technique for using power iteration algorithms
to recover multiple eigenvalues.
Sometimes, this may be preferable to using a single higher-rank invocation of
Alecton (for example, we may not know a priori how many eigenvectors we want).
We outline this technique as Algorithm \ref{algAlectonOAAT}.
\begin{algorithm}[h]
  \caption{Alecton One-at-a-time}
  \begin{algorithmic}
  \label{algAlectonOAAT}
    \REQUIRE A sampling distribution $\mathcal{A}$
    \STATE $\mathcal{A}_1 \rightarrow \mathcal{A}$
    \FOR{$i = 1$ \TO $p$}
    \STATE \(\triangleright\) Run rank-1 Alecton to produce output $y_i$.
    \STATE $y_i \rightarrow \mathrm{Alecton}_{p=1}(\mathcal{A}_i)$
    \STATE Generate sampling distribution $\mathcal{A}_{i+1}$ such that,
      if $\tilde A'$ is sampled from $\mathcal{A}_{i+1}$ and $\tilde A$ is
      sampled from $\mathcal{A}_i$, $\Exv{\tilde A'} = \Exv{\tilde A} - y_i y_i^T$.
    \ENDFOR
    \RETURN $\sum_{i=1}^p y_i y_i^T$
  \end{algorithmic}
\end{algorithm}
This strategy allows us to recover
the largest $p$ eigenvectors of $A$ using $p$ executions of Alecton.
If the eigenvalues of the matrix are independent of $n$ and $p$, we will be able
to accomplish this in $O(\epsilon^{-1} p n \log n)$ total steps.

\section{Experiments}

\begin{figure*}[t]
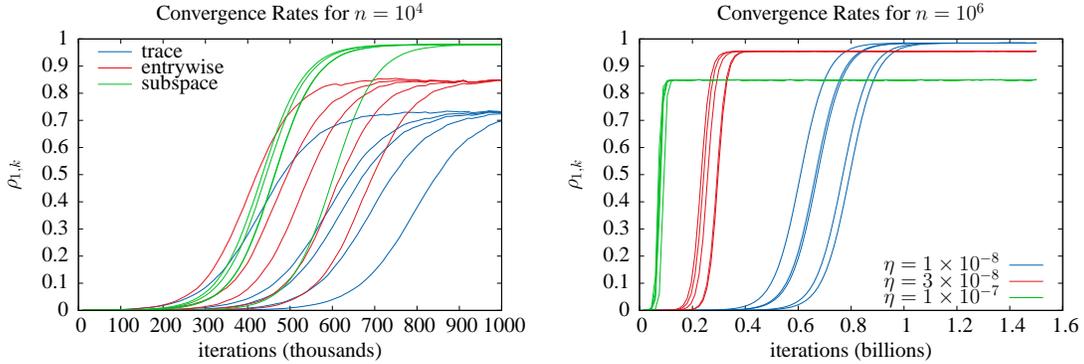
%
\centering
\subfigure[Angular convergence of three distributions on a synthetic dataset with $\eta = 10^{-5}$.]{%
\centering%
\resizebox{!}{.30\textwidth}{%
\Large \input{plot10k.tex}%
}%
\label{plot10k}%
}\quad%
\subfigure[Angular convergence of entrywise sampling on a large synthetic dataset for different step sizes.]{%
\centering%
\resizebox{!}{.30\textwidth}{%
\Large \input{plot1m.tex}%
}%
\label{plot1m}}
\label{plotConvergence} 
\caption{Convergence occurs in $O(n \log n)$ steps.}
\end{figure*}

\begin{figure}[h]
  \centering%
  \resizebox{!}{.30\textwidth}{
  \Large \input{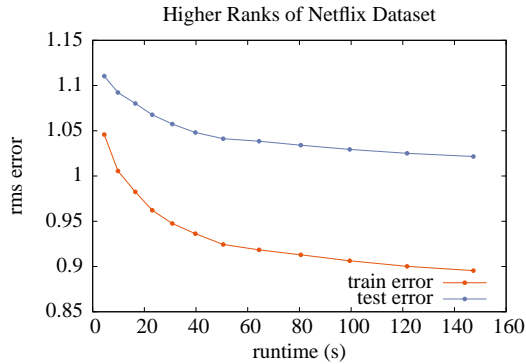}
  }
  \caption{RMS errors over Netflix dataset~\cite{funk} for higher-rank recovery.  Each
  point represents an additional recovered eigenvector found with
  Alecton One-at-a-time.}
  \label{plotNetflixMulti}
\end{figure}

We experimentally verify our main claim, that Alecton does converge quickly
for practical datasets.

All experiments were run on a machine with a single twelve-core
socket (Intel Xeon E5-2697, 2.70GHz), and 256 GB of shared memory.
All were written in C++, excepting the Netflix Prize problem
experiment, which was written in Julia.
No data was collected for the radial phase of
Alecton, since the performance of averaging
is already well understood.

The first experiments were run on randomly-generated rank-$10$ data
matrices $A \in \R^{n \times n}$.  Each was generated by selecting a random
orthogonal matrix $U \in \R^{n \times n}$, then independently selecting a
diagonal matrix $\Lambda$ with $10$ positive nonzero eigenvalues,
and constructing  $A = U \Lambda U'$.  Figure \ref{plot10k} illustrates
the convergence of Alecton with $p = q = 1$
using three sampling distributions on 
datasets with $n = 10^4$.  We ran Alecton starting from five
random initial values; the different plotted trajectories illustrate
how convergence time can depend on the initial value.

Figure \ref{plot1m} illustrates the performance of Alecton ($p = q = 1$ again)
on a larger dataset with $n = 10^6$ as the step size parameter
$\eta$ is varied.  As we would expect, a smaller value of
$\eta$ yields slower, but more accurate convergence.  Also notice
that the smaller the value of $\eta$, the more the initial value
seems to affect convergence time.

Figure \ref{plotNetflixMulti} demonstrates convergence results on real
data from the Netflix Prize problem.  This problem involves
recovering a matrix with 480,189 columns and 17,770 rows from a training
dataset containing 110,198,805 revealed entries.
We used the rectangular entrywise distribution described above,
then ran Alecton with $\eta = 10^{-12}$ and $p = q = 1$ for ten million iterations to
recover the most significant singular vector.
Next, we used Algorithm \ref{algAlectonOAAT}
to recover additional singular vectors of the matrix, up
to a maximum of $p = 12$.  The absolute runtime and RMS errors after the
recovery of each subsequent eigenvector
are plotted in Figure \ref{plotNetflixMulti}.  This plot illustrates that
the runtime of the one-at-a-time algorithm does not
increase disastrously as the number of recovered eigenvectors expands.

\subsection{Discussion}
The Hogwild! algorithm~\cite{hogwild} is a parallel, lock-free version
of stochastic gradient descent that has been shown to perform similarly
to sequential SGD on convex problems, while allowing for a good parallel speedup.
It is an open question whether a Hogwild! version
of Alecton for non-convex problems converges with a good rate, but we
are optimistic that it will.

\section{Conclusion}

This paper exhibited Alecton, a stochastic gradient descent
algorithm applied to
a non-convex low-rank factorized problem; it is similar to the
algorithms used in practice to solve a wide variety of problems.
We prove that Alecton converges globally, and provide a rate
of convergence.  We do not require any special initialization step
but rather initialize randomly.  Furthermore,
our result depends only on the variance of the samples, and therefore
holds under broad sampling conditions that include both matrix
completion and matrix sensing, and is also able to take noisy samples into
account.  We show these results using a martingale-based technique that is
novel in the space of non-convex optimization, and we are optimistic
that this technique can be applied to other problems in the future.

\ifdefined\isaccepted

\section*{Acknowledgments} 

Thanks to Ben Recht, Mahdi Soltanolkotabi, Joel Tropp, Kelvin Gu,
and Madeleine Udell for helpful
conversations. Thanks also to Ben Recht and Laura Waller for datasets.

The authors acknowledge the support of the Defense Advanced
Research Projects Agency (DARPA) XDATA Program under No. FA8750-12-2-0335 and
DEFT Program under No. FA8750-13-2-0039, DARPA’s MEMEX program, the National
Science Foundation (NSF) CAREER Award under No. IIS-1353606 and EarthCube Award
under No. ACI-1343760, the Sloan Research Fellowship, 
the Office of Naval Research (ONR) under awards No.
N000141210041 and No. N000141310129, the Moore
Foundation, American Family Insurance, Google, and Toshiba.  Additionally,
the authors acknowledge: 
DARPA Contract-Air Force, Deliteful DeepDive: Domain-Specific Indexing
and Search for the Web, FA8750-14-2-0240;
Army Contract AHPCRC W911NF-07-2-0027-1;
NSF Grant, BIGDATA: Mid-Scale: DA: Collaborative Research:
Genomes Galore - Core Techniques, Libraries, and Domain Specific
Languages for High-Throughput DNA Sequencing, IIS-1247701;
NSF Grant, SHF: Large: Domain Specific Language Infrastructure
for Biological Simulation Software, CCF-1111943;
Dept. of Energy- Pacific Northwest National Lab (PNNL) - Integrated
Compiler and Runtime Autotuning Infrastructure for Power, Energy and
Resilience-Subcontract 108845;
NSF Grant – EAGER- XPS:DSD:Synthesizing Domain Specific Systems-CCF-1337375;
Stanford PPL affiliates program, Pervasive Parallelism Lab:  Oracle,
NVIDIA, Huawei, SAP Labs.
The authors also acknowledge additional support from Oracle.
Any opinions,
findings, and conclusions or recommendations expressed in this material are
those of the authors and do not necessarily reflect the views of DARPA, AFRL,
NSF, ONR, or the U.S. government.

\else

\fi




\bibliographystyle{plainnat} 
\bibliography{references}





\appendix

\section{Negative Results}
\label{ssCounterexamples}

\paragraph*{Divergence Example}
Here, we observe what happens when we choose a constant step size for stochastic
gradient descent for quartic objective functions.  Consider the simple optimization problem
of minimizing
\[
  f(x) = \frac{1}{4} x^4.
\]
This function will have gradient descent update rule
\[
  x_{k+1} = x_k - \alpha_k x_k^3
    = \left(1 - \alpha_k x_k^2 \right) x_k.
\]
We now prove that, for any reasonable step size rule chosen independently of $x_k$,
there is some initial condition such that this iteration diverges to infinity.
\begin{proposition}
  Assume that we iterate using the above rule, for some choice of $\alpha_k$ that is not
  super-exponentially decreasing; that is, for some $C > 1$ and some $\alpha > 0$,
  $\alpha_k \ge \alpha C^{-2k}$ for all $k$.  Then, if
  $x_0^2 \ge \alpha^{-1} (C + 1)$, for all $k$
  \[
    x_k^2 > \alpha^{-1} C^{2k} (C + 1).
  \]
\end{proposition}
\begin{proof}
  We will prove this by induction.  The base case follows directly from the assumption,
  while under the inductive case, if the proposition is true for $k$, then
  \[
    \alpha_k x_k^2 \ge \alpha C^{-2k} \alpha^{-1} C^{2k} (C + 1)
      = C + 1.
  \]
  Therefore,
  \begin{dmath*}
    x_{k+1}^2 = \left(\alpha_k x_k^2 - 1 \right)^2 x_k^2
      \ge C^2 x_k^2
      \ge C^2 \alpha^{-1} C^{2k} (C + 1)
      = \alpha^{-1} C^{2(k + 1)} (C + 1).
  \end{dmath*}
  This proves the statement.
\end{proof}
This proof shows that, for some choice of $x_0$, $x_k$ will diverge to infinity exponentially
quickly.  Furthermore, no reasonable choice of $\alpha_k$  will be able to halt this increase
for all initial conditions.  We can see the effect of this in stochastic gradient
descent as well, where there is always some probability that, due to an unfortunate series
of gradient steps, we will enter the zone in which divergence occurs.  On the other hand,
if we chose step size $\alpha_k = \gamma_k x_k^{-2}$, for some $0 < \gamma_k < 2$, then
\[
  x_{k+1} = \left(1 - \gamma_k \right) x_k,
\]
which converges for all starting values of $x_k$.  This simple example is what motivates
us to take $\norm{Y_k}$ into account when choosing the step size for Alecton.

\paragraph{Global Convergence Counterexample}
We now exhibit a particular problem for which SGD on a low-rank factorization
doesn't converge to the global optimum for a particular starting point.
Let matrix $A \in \R^{2 \times 2}$ be the diagonal matrix with diagonal entries
$4$ and $1$.  Further, let's assume that we are trying to minimize the
expected value of the decomposed rank-1 objective function
\[
  \tilde f(y) = \normf{\tilde A - y y^T} = \norm{y}^4 - 2 y^T \tilde A y + \normf{\tilde A}^2.
\]
If our stochastic samples satisfy $\tilde A = A$ (i.e. we use a 
perfect sampler), then the SGD update rule is
\[
  y_{k+1} = y_k - \alpha_k \nabla \tilde f(y_k)
    = y_k - 4 \alpha_k \left(y_k \norm{y_k}^2 - A y_k \right).
\]
Now, we know that $e_1$ is the most significant eigenvector of $A$, and that
$y = 2 e_1$ is the global solution to the problem.  However,
\begin{align*}
  e_1^T y_{k+1}
  &=
  e_1^T y_k - 4 \alpha_k \left(e_1^T y_k \norm{y_k}^2 - e_1^T A y_k \right) \\
  &=
  \left(1 - 4 \alpha_k \left(\norm{y_k}^2 - 4 \right)\right) e_1^T y_k
\end{align*}.
This implies that if $e_1^T y_0 = 0$, then $e_1^T y_k = 0$ for all $k$, which means
that convergence to the global optimum cannot occur.  This illustrates that
global convergence does not occur for all manifold optimization problems
using a low-rank factorization and for all starting points.

\paragraph{Constraints Counterexample}
We might think that our results can be generalized to give $O(n \log n)$
convergence of low-rank factorized problems with arbitrary constraints.
Here, we show that
this will not work for all problems by encoding an NP-complete problem as a
constrained low-rank optimization problem.

For any graph with node set $N$ and edge set $E$, the MAXCUT problem
on the graph requires us to solve
\[
  \begin{array}{ll}
    \mbox{minimize} & \sum_{(i,j) \in E} y_i y_j \\
    \mbox{subject to} & y_i \in \{-1, 1\}.
  \end{array}
\]
Equivalently, if we let $A$ denote the edge-matrix of the graph, we can represent this as
a matrix problem~\cite{homer1997,goemans1995}
\[
  \begin{array}{ll}
    \mbox{minimize} & y^T A y \\
    \mbox{subject to} & y_i \in \{-1, 1\}.
  \end{array}
\]
We relax this problem to
\[
  \begin{array}{ll}
    \mbox{minimize} & y^T A y \\
    \mbox{subject to} & -1 \le y_i \le 1.
  \end{array}
\]
Since the diagonal of $A$ is zero, if we fix all but one of the
entries of $y$, the objective function will have an affine dependence on
that entry.  In particular, this means that a global minimum of the problem
must occur on the boundary where $y_i \in \{-1, 1\}$, which implies
that this problem has the same global solution as the original MAXCUT
problem.  Furthermore, for sufficiently large values of $\sigma$,
the problem
\[
  \begin{array}{ll}
    \mbox{minimize} & \norm{y}^4 + 2 \sigma y^T A y + \sigma^2 \normf{A}^2 \\
    \mbox{subject to} & -1 \le y_i \le 1
  \end{array}
\]
will also have the same solution.  But, this problem is in the same form
as a low-rank factorization of
\[
  \begin{array}{ll}
    \mbox{minimize} & \normf{X + \sigma A}^2 \\
    \mbox{subject to} & X_{ii} \le 1, X \succeq 0, \rank{X} = 1
  \end{array}
\]
where $X = y y^T$.  Since MAXCUT is NP-complete, it can't possibly be the
case that SGD applied to this low-rank factorized problem converges quickly
to the global optimum, because that would imply an efficient solution to this
NP-complete problem.  This suggests that care will be needed when analyzing
problems with constraints, in order to exclude these sorts of cases.

\section{Comparison with Other Methods}
\label{ssComparison}
There are several other algorithms that solve similar matrix recover problems
in the literature.  In Table \ref{tableComparison}, we list some other algorithms, and their convergence
rates, in terms of both number of samples required (sampling complexity)
and number of iterations performed (computational complexity).  For this table,
the data is assumed to be of dimension $n$, and the rank (where applicable) is
assumed to be $p$.  (In order to save space, factors of $\log \log \epsilon^{-1}$
have been omitted from some formulas.)
\begin{figure*}[t]
\label{tableComparison}
\begin{center}
  {\tabulinesep=1.2mm
  \begin{tabu}{ | l | c | >{\centering}m{2.5cm} | >{\centering}m{2.5cm} | }
    \hline
    Algorithm & Sampling Scheme & \multicolumn{2}{c |}{Complexity} \\
    & & Sampling & Computational \\ \hline
    Alecton & Any & \multicolumn{2}{c |}{$O(\epsilon^{-1} p^3 n \log n)$} \\ \hline
    SVD & Various & $o(p n)$ & $O(n^3)$ \\ \hline
    Spectral Matrix Completion~\citesec{keshavan2010} & Elementwise & $o(p n)$ & $O(p^2 n \log n)$ \\ \hline
    PhaseLift~\citesec{phaselift2014} & Phase Retrieval & $o(n)$ & $O(\epsilon^{-1} n^3)$ \\ \hline
    Alternating Minimization~\citesec{netrapalli2013} & Phase Retrieval & $o(n \log (\epsilon^{-1}))$ & $O(n^2 \log^2(\epsilon^{-1}))$ \\ \hline
    Wirtinger Flow~\citesec{candes2014phase} & Phase Retrieval & $o(n \log^2 n)$ & $O(p n \log(\epsilon^{-1}))$ \\
    \hline
  \end{tabu}}
\end{center}
\end{figure*}

\section{Proofs of Main Results}
\label{ssaProofs}
In this appendix, we provide rigorous definitions and 
detail the proof outlined in Section \ref{ssMartingaleTechnique}.

\subsection{Definitions}
\citet{fleming1991} provide the following definitions of filtration and martingale.
We state the definitions adapted to the discrete-time case.
\begin{definition}[Filtration]
  Given a measurable probability space $(\Omega, \F)$, a \emph{filtration} is
  a sequence of sub-$\sigma$-algebras $\left\{ \F_t \right\}$ for $t \ge 0$,
  such that for all $s \le t$,
  \[
    \F_s \subset \mathcal{F}_t.
  \]
\end{definition}
That is, if an event $A$ is in $\F_s$, and $t \ge s$, then $A$ is also in $\F_t$.
This definition encodes the monotonic increase in available information over time.
\begin{definition}[Martingale]
  Let $\left\{X_t\right\}$ be a stochastic process and $\left\{ \F_t \right\}$ be
  a filtration over the same probability space.  Then $X$ is called a
  \emph{martingale} with respect to the filtration if for every $t$, $X_t$
  is $\F_t$-measurable, and
  \begin{equation}
    \label{eqnMartingaleDefn}
    \Exvc{X_{t+1}}{\F_t} = X_t.
  \end{equation}
  We call $X$ a \emph{submartingale} if the same conditions hold, except
  (\ref{eqnMartingaleDefn}) is replaced with
  \[
    \Exvc{X_{t+1}}{\F_t} \ge X_t.
  \]
  We call $X$ a \emph{supermartingale} if the same conditions hold, except
  (\ref{eqnMartingaleDefn}) is replaced with
  \[
    \Exvc{X_{t+1}}{\F_t} \le X_t.
  \]
\end{definition}

\subsection{Preliminaries}
In addition to the quantities used in the statement of Theorem \ref{thmAlecton}, we let
\[
  W
  =
  \gamma n^{-1} p^{-2} q I
  +
  (1 - \gamma n^{-1} p^{-2} q) U,
\]
and define sequences $\tau_k$ and $\phi_k$ as
\[
  \tau_k
  =
  \frac{
    \Det{Y_k^T U Y_k}
  }{
    \Det{Y_k^T W Y_k}
  },
\]
and
\[
  \phi_k
  =
  \trace{
    I
    -
    Y_k^T U Y_k \left( Y_k^T W Y_k \right)^{-1}
  }.
\]
This agrees with the definition of $\tau_k$ stated in the body of
the paper.  Using this sequence, we define the failure event
$f_k$ as the event that occurs when
\begin{equation}
  \label{eqnFailCondition}
  \tau_k \le \frac{1}{2}.
\end{equation}
We recall that we defined the success event at time $k$ as the event
that, for all $z \in \R^p$,
\[
  \frac{\norm{U Y_k z}^2}{\norm{Y_k z}^2} \ge 1 - \epsilon.
\]
Finally, we define $T$, the stopping time, to be the first time 
at which either the success event or the failure event occurs.

Now, we state some lemmas we will need in the following proofs.  We defer
proofs of the lemmas themselves to Appendix \ref{ssLemmaProofs}.  First, 
we state a lemma about quadratic rational functions that we will need in the
next section.
\begin{lemma}[Quadratic rational lower bound]
  \label{lemma-quadlb}
  For any $a$, $b$, $c$, and $d$ in $\R$, if 
  $1 + by + cy^2 > 0$ and $1 + ay + dy^2 \ge 0$ for all $y$, then
  for all $x \in \R$,
  \[
    \frac{
      1 + ax + dx^2
    }
    {
      1 + bx + cx^2
    }
    \ge
    1 + (a - b) x - c x^2.
  \]
\end{lemma}

Next, a lemma about the expected initial value of $\tau$:
\begin{lemma}
  \label{lemmaTau0}
  If we initialize $Y_0$ uniformly as in the Alecton algorithm, then
  \[
    \Exv{\tau_0}
    \ge
    1 - \frac{1}{2} Z_p(\gamma).
  \]
\end{lemma}

Next, a lemmas that bounds a determinant expression.
\begin{lemma}
  \label{lemmaDetRank1UpdateLUBound}
  For any $B \in \R^{n \times n}$, $Y \in \R^{n \times m}$, and any
  symmetric positive-semidefinite $Z \in \R^{n \times n}$, if
  either $B$ is rank-1 or $m = 1$, then
  \begin{align*}
    &\Det{Y^T (I + B)^T Z (I + B) Y} \\
    &\hspace{1em}\ge
    \Det{Y^T Z Y} \left(\trace{Y (Y^T Z Y)^{-1} Y^T Z B} + 1\right)^2
  \end{align*}
  and
  \begin{align*}
    &\Det{Y^T (I + B)^T Z (I + B) Y} \\
    &\hspace{1em}\le
    \Det{Y^T Z Y} \Big( 1 + 2 \trace{Y (Y^T Z Y)^{-1} Y^T Z B} \\
    &\hspace{2em}+
    \trace{Y (Y^T Z Y)^{-1} Y^T B^T Z B} \Big).
  \end{align*}
\end{lemma}

Next, a lemma that bounds $\tau$ in the case that the success
condition does not occur.
\begin{lemma}
  \label{lemmaTauNoSuccess}
  If we run Alecton, and at timestep $k$, the success condition
  does not hold, then
  \[
    \tau_k
    \le
    1 - \gamma n^{-1} p^{-2} q \epsilon.
  \]
\end{lemma}

Finally, a lemma that relates $\phi$ and $\tau$.
\begin{lemma}
  \label{lemmaPhiTau}
  Using the definitions above, for all $k$,
  \[
    \phi_k \ge 1 - \tau_k.
  \]
\end{lemma}

\subsection{Main Proofs}

We now proceed to prove Theorem \ref{thmAlecton} in six steps, as outlined in
Section \ref{ssMartingaleTechnique}.
\begin{itemize}
  \item First, we prove Lemma \ref{lemmaDMB}, the \emph{dominant mass bound lemma},
    which bounds $\Exvc{\tau_{k+1}}{\F_k}$ from below by a quadratic function
    of the step size $\eta$.
  \item We use this to prove Lemma \ref{lemmaConstEta}, which establishes the result
    stated in (\ref{taukeqn}).
  \item We use the optional stopping theorem to prove Lemma \ref{lemmaFailureProb},
    which bounds the probability of a failure event occurring before success.
  \item We use the optional stopping theorem again to prove Lemma
    \ref{lemmaStoppingTime}, which bounds the expected time until either a failure
    or success event occurs.
  \item We use Markov's inequality and the union bound to bound the angular failure
    probability of Theorem \ref{thmAlecton}.
  \item Finally, we prove the radial phase result stated in Theorem \ref{thmAlecton}.
\end{itemize}

\begin{lemma}[Dominant Mass Bound]
  \label{lemmaDMB}
  If we run Alecton under the conditions of Theorem \ref{thmAlecton}, then
  for any $k$,
  \begin{align*}
    \Exvc{\tau_{k+1}}{\F_k}
    \ge
    \tau_k \Big(
      1 
      &+ 
      2 \eta \left(
        \Delta
        - \eta \sigma_a^2 \gamma^{-1} n p^2
      \right) (1 - \tau_k) \\
    &\hspace{2em}-
      \eta^2 \sigma_a^2 p (q + 1)
    \Big).
  \end{align*}
\end{lemma}
\begin{proof}
  From the definition of $\tau$, at the next timestep we will have
  \begin{align*}
    \tau_{k+1}
    &=
    \frac{
      \Det{Y_{k+1}^T U Y_{k+1}}
    }{
      \Det{Y_{k+1}^T W Y_{k+1}}
    } \\
    &=
    \frac{
      \Det{
        Y_k^T \left(I + \eta \tilde A_k \right)^T
        U
        \left(I + \eta \tilde A_k \right) Y_k
      }
    }{
      \Det{
        Y_k^T \left(I + \eta \tilde A_k \right)^T
        W
        \left(I + \eta \tilde A_k \right) Y_k
      }
    }.
  \end{align*}
  Now, since our instance of Alecton satisfies the rank condition, either
  $\tilde A_k$ is rank-1 or $p = 1$.  Therefore, we can apply Lemma
  \ref{lemmaDetRank1UpdateLUBound} to these determinant quantities.
  In order to produce a lower bound on $\tau_{k+1}$, we will
  apply lower bound to the numerator and the upper bound to
  the denominator.
  If we let $B_k = Y_k (Y_k^T U Y_k)^{-1} Y_k^T$,
  and $C_k = Y_k (Y_k^T W Y_k)^{-1} Y_k^T$, then this results in
  \begin{align*}
    &\tau_{k+1} \ge
    \frac{\Det{Y_k^T U Y_k}}{\Det{Y_k^T W Y_k}} \\
    &\hspace{2em} \cdot
    \frac{
      \left(
        1
        +
        \eta \trace{B_k U \tilde A_k}
      \right)^2
    }{
      1
      +
      2 \eta \trace{C_k W \tilde A_k}
      +
      \eta^2 \trace{C_k \tilde A_k^T W \tilde A_k}
    }.
  \end{align*}
  Next, we apply Lemma \ref{lemma-quadlb}, which results in
  \begin{align*}
    \tau_{k+1}
    &\ge
    \tau_k \Big( 1 +
    2 \eta \left(
      \trace{B_k U \tilde A_k}
      -
      \trace{C_k W \tilde A_k}
    \right) \\
    &\hspace{4em}-
    \eta^2 \trace{C_k \tilde A_k^T W \tilde A_k}
    \Big) \\
    &\ge
    \tau_k \left(1 + 2 \eta R_k + \eta^2 Q_k \right),
  \end{align*}
  for sequences $R_k$ and $Q_k$.
  Now, we investigate the expected values of these sequences.
  First, since the estimator has $\Exvc{\tilde A_k}{\F_k} = A$, the expected value of $R_k$ is
  \begin{align*}
    \Exvc{R_k}{\F_k}
    &=
    \trace{B_k U A}
    -
    \trace{C_k W A} \\
    &= 
    \trace{(B_k - C_k) U A} \\
    &\hspace{2em}-
    \gamma n^{-1} p^{-2} q \trace{C_k (I - U) A}.
  \end{align*}
  Now, since $U$ commutes with $A$, we will have that
  \[
    U A \succeq \lambda_q U,
  \]
  and similarly
  \[
    (I - U) A \preceq \lambda_{q+1} (I - U).
  \]
  Applying this results in
  \begin{align*}
    \Exvc{R_k}{\F_k}
    &\ge
    \trace{B_k U A}
    -
    \trace{C_k W A} \\
    &= 
    \lambda_q \trace{(B_k - C_k) U} \\
    &\hspace{2em}-
    \lambda_{q+1} \gamma n^{-1} p^{-2} q \trace{C_k (I - U)}.
  \end{align*}
  Now, we first notice that
  \begin{align*}
    \trace{(B_k - C_k) U}
    &=
    \trace{I - Y_k U Y_k^T (Y_k^T W Y_k)^{-1}} \\
    &=
    \phi_k.
  \end{align*}
  We also notice that
  \begin{align*}
    &\gamma n^{-1} p^{-2} q \trace{C_k (I - U)}
    =
    \trace{C_k (W - U)} \\
    &\hspace{2em}=
    \trace{I - Y_k U Y_k^T (Y_k^T W Y_k)^{-1}} \\
    &\hspace{2em}=
    \phi_k.
  \end{align*}
  It therefore follows that
  \begin{align*}
    \Exvc{R_k}{\F_k}
    &\ge
    (\lambda_q - \lambda_{q+1}) \phi_k \\
    &=
    \Delta \phi_k.
  \end{align*}
  Next, the expected value of $Q_k$ is
  \begin{align*}
    \Exvc{Q_k}{\F_k}
    &=
    \trace{C_k \Exv{\tilde A_k^T W \tilde A_k}}.
  \end{align*}
  Since our instance of Alecton satisfies the variance condition,
  and $W$ commutes with $A$,
  \begin{align*}
    \Exvc{Q_k}{\F_k}
    &\le
    \sigma_a^2 \trace{W} \trace{C_k}.
  \end{align*}
  We notice that
  \begin{align*}
    \trace{C_k}
    &=
    \trace{C_k \left(
      W
      +
      (1 - \gamma n^{-1} p^{-2} q) (I - U)
    \right)} \\
    &=
    p
    +
    (1 - \gamma n^{-1} p^{-2} q) \trace{C_k (I - U)} \\
    &\le
    p
    +
    \trace{C_k (I - U)}.
  \end{align*}
  By the logic above,
  \begin{align*}
    \trace{C_k}
    &\le
    p
    +
    \gamma^{-1} n p^2 q^{-1} \phi_k.
  \end{align*}
  Also,
  \begin{align*}
    \trace{W}
    &=
    \trace{\gamma n^{-1} p^{-2} q I
    +
    (1 - \gamma n^{-1} p^{-2} q) U} \\
    &=
    \gamma p^{-2} q
    +
    q - \gamma n^{-1} p^{-2} q^2 \\
    &\ge
    q + 1
  \end{align*}
  and therefore, since $\trace{W} \le q + 1$,
  \begin{align*}
    \Exvc{Q_k}{\F_k}
    &\le
    \sigma_a^2 (q + 1) \left(
      p
      +
      \gamma^{-1} n p^2 q^{-1} \phi_k
    \right).
  \end{align*}
  Substituting these in results in
  \begin{dmath*}
    \Exvc{\tau_{k+1}}{\F_k}
    \ge
    \tau_k \left(
      1 
      + 
      2 \eta \Delta \phi_k
      -
      \eta^2 \left(
        \sigma_a^2 p (q + 1)
        + 
        \sigma_a^2 \gamma^{-1} n p^2 (q + 1) q^{-1} \phi_k
      \right)
    \right)
    =
    \tau_k \left(
      1 
      + 
      \eta \left(
        2 \Delta
        - \eta \sigma_a^2 \gamma^{-1} n p^2 (q + 1) q^{-1}
      \right) \phi_k
      -
      \eta^2 \sigma_a^2 p (q + 1)
    \right)
    \ge
    \tau_k \left(
      1 
      + 
      2 \eta \left(
        \Delta
        - \eta \sigma_a^2 \gamma^{-1} n p^2
      \right) \phi_k
      -
      \eta^2 \sigma_a^2 p (q + 1)
    \right).
  \end{dmath*}
  Finally, since for our chosen value of $\gamma$,
  \[
    \Delta
    >
    \eta \sigma_a^2 \gamma^{-1} n p^2,
  \]
  we can apply Lemma \ref{lemmaPhiTau}, which produces
  \begin{align*}
    \Exvc{\tau_{k+1}}{\F_k}
    \ge
    \tau_k \Big(
      1 
      &+ 
      2 \eta \left(
        \Delta
        - \eta \sigma_a^2 \gamma^{-1} n p^2
      \right) (1 - \tau_k) \\
    &\hspace{2em}-
      \eta^2 \sigma_a^2 p (q + 1)
    \Big).
  \end{align*}
  This is the desired expression.
\end{proof}

\begin{lemma}
  \label{lemmaConstEta}
  If we run Alecton under the conditions of Theorem \ref{thmAlecton},
  then for any time $k$ at which neither the success event nor the
  failure event occur,
  \[
    \Exvc{\tau_{k+1}}{\F_k}
    \ge
    \tau_k \left(
      1
      +
      \eta \Delta (1 - \tau_k)
    \right).
  \]
\end{lemma}
\begin{proof}
  From the result of Lemma \ref{lemmaDMB},
  \begin{dmath*}
    \Exvc{\tau_{k+1}}{\F_k}
    \ge
    \tau_k \left(
      1 
      + 
      2 \eta \left(
        \Delta
        - \eta \sigma_a^2 \gamma^{-1} n p^2
      \right) (1 - \tau_k)
      -
      \eta^2 \sigma_a^2 p (q + 1)
    \right)
    =
    \tau_k \left(
      1 
      + 
      \eta \Delta (1 - \tau_k)
      +
      \eta \left(
        \Delta
        - 2 \eta \sigma_a^2 \gamma^{-1} n p^2
      \right) (1 - \tau_k)
      -
      \eta^2 \sigma_a^2 p (q + 1)
    \right)
    =
    \tau_k \left(
      1 
      + 
      \eta \Delta (1 - \tau_k)
      +
      \eta S_k,
    \right)
  \end{dmath*}
  for sequence $S_k$.
  Now, it can be easily verified that we chose $\gamma$ such that
  \[
    \Delta
    \ge
    2 \eta \sigma_a^2 \gamma^{-1} n p^2,
  \]
  and so it follows that, by Lemma \ref{lemmaTauNoSuccess},
  \begin{align*}
    S_k
    &=
    \left(
      \Delta
      - 2 \eta \sigma_a^2 \gamma^{-1} n p^2
    \right) (1 - \tau_k)
    -
    \eta \sigma_a^2 p (q + 1) \\
    &\ge
    \left(
      \Delta
      - 2 \eta \sigma_a^2 \gamma^{-1} n p^2
    \right) \gamma n^{-1} p^{-2} q \epsilon
    -
    \eta \sigma_a^2 p (q + 1) \\
    &=
    \Delta \gamma n^{-1} p^{-2} q \epsilon
    -
    2 \eta \sigma_a^2 q \epsilon
    -
    \eta \sigma_a^2 p (q + 1) \\
    &\ge
    \Delta \gamma n^{-1} p^{-2} q \epsilon
    -
    2 \eta \sigma_a^2 q (p + \epsilon).
  \end{align*}
  If we substitute the value of $\gamma$,
  \[
    \gamma
    =
    \frac{
      2 n \sigma_a^2 p^2 (p + \epsilon)
    }{
      \Delta \epsilon
    }
    \eta.
  \]
  then we arrive at
  \[
    S_k \ge 0.
  \]
  Substituting this in to our original expression produces
  \[
    \Exvc{\tau_{k+1}}{\F_k}
    \ge
    \tau_k \left(
      1
      +
      \eta \Delta (1 - \tau_k)
    \right),
  \]
  as desired.
\end{proof}

\begin{lemma}[Failure Probability Bound]
  \label{lemmaFailureProb}
  If we run Alecton under the conditions of Theorem \ref{thmAlecton},
  then the probability that the failure event will occur before the
  success event is
  \[
    \Prob{f_T}
    \le
    Z_p(\gamma).
  \]
\end{lemma}
\begin{proof}
  To prove this, we use the stopping time $T$, which we defined as
  the first time at which either the success event or failure event occurs.
  First, if $k < T$, it follows that neither success nor failure have
  occurred yet, so
  we can apply Lemma \ref{lemmaConstEta}, which results in
  \[
    \Exvc{\tau_{k+1}}{\F_k}
    \ge
    \tau_k \left(
      1
      +
      \eta \Delta (1 - \tau_k)
    \right).
  \]
  Therefore $\tau_k$ is a supermartingale for $k < T$.  So, we can apply
  the optional stopping theorem, which produces
  \[
    \Exv{\tau_0} \le \Exv{\tau_T}.
  \]
  So, by the law of total expectation,
  \[
    \Exv{\tau_0}
    \le
    \Exvc{\tau_T}{f_T} \Prob{f_T}
    +
    \Exvc{\tau_T}{\lnot f_T} \Prob{\lnot f_T},
  \]
  where $f_T$ is the failure event at time $T$.
  Applying the definition of the failure event from (\ref{eqnFailCondition}),
  \[
    \Exv{\tau_0}
    \le
    \frac{1}{2} \Prob{f_T}
    +
    1 \left(1 - \Prob{f_T} \right).
  \]
  Therefore, solving for $\Prob{f_T}$,
  \[
    \Prob{f_T}
    \le
    2 \left(1 - \Exv{\tau_0} \right).
  \]
  Now applying Lemma \ref{lemmaTau0},
  \[
    \Prob{f_T}
    \le
    2 \left(1 - \left(1 - \frac{1}{2} Z_p(\gamma) \right) \right)
    =
    Z_p(\gamma),
  \]
  as desired.
\end{proof}

\begin{lemma}[Stopping Time Expectation]
  \label{lemmaStoppingTime}
  If we run Alecton under the conditions of Theorem \ref{thmAlecton},
  then the expected value of the stopping time $T$ will be
  \[
    \Exv{T}
    \le
    \frac{4 n \sigma_a^2 p^2 (p + \epsilon)}{\Delta^2 \gamma \epsilon}
    \log\left( \frac{n p^2}{\gamma q \epsilon} \right).
  \]
\end{lemma}
\begin{proof}
  First, as above if $k < T$, we can apply Lemma \ref{lemmaConstEta}, which results in
  \begin{align*}
    \Exvc{\tau_{k+1}}{\F_k}
    &\ge
      \tau_k \left(
        1
        +
        \eta \Delta \left( 1 - \tau_k \right)
      \right) \\
    &=
    \tau_k
    +
    \eta \Delta \tau_k \left( 1 - \tau_k \right),
  \end{align*}
  and so
  \[
    \Exvc{1 - \tau_{k+1}}{\F_k}
    \le
    (1 - \tau_k) \left(
      1
      -
      \eta \Delta \tau_k
    \right).
  \]
  Now, if $k < T$, then since failure hasn't occurred yet, $\tau_k > \frac{1}{2}$.
  So,
  \[
    \Exvc{1 - \tau_{k+1}}{\F_k}
    \le
    (1 - \tau_k) \left(
      1
      -
      \frac{1}{2} \eta \Delta
    \right).
  \]
  Now, since the logarithm function is concave, by Jensen's inequality we have
  \[
    \Exvc{\log\left(1 - \tau_{k+1}\right)}{\F_k}
    \le
    \log \Exvc{1 - \tau_{k+1}}{\F_k},
  \]
  and thus by transitivity,
  \begin{align*}
    \Exvc{\log\left(1 - \tau_{k+1}\right)}{\F_k}
    &\le
    \log(1 - \tau_k) 
    +
    \log \left(
      1
      -
      \frac{1}{2} \eta \Delta
    \right) \\
    &\le
    \log(1 - \tau_k) 
    -
    \frac{1}{2} \eta \Delta.
  \end{align*}
  Now, we define a new process $\psi_k$ as
  \[
    \psi_k = \log(1 - \tau_k) + \frac{1}{2} \eta \Delta k.
  \]
  Using this definition, for $k < T$,
  \begin{align*}
    \Exvc{\psi_{k+1}}{\F_k}
    &=
    \Exvc{\log(1 - \tau_{k+1})}{\F_k} + \frac{1}{2} \eta \Delta (k+1) \\
    &\le
    \log(1 - \tau_k)
    -
    \frac{1}{2} \eta \Delta
    +
    \frac{1}{2} \eta \Delta (k+1) \\
    &=
    \log(1 - \tau_k) 
    +
    \frac{1}{2} \eta \Delta k \\
    &=
    \psi_k,
  \end{align*}
  so $\psi_k$ is a supermartingale for $k < T$.  We can therefore apply the optional stopping
  theorem, which states that
  \[
    \Exv{\log(1 - \tau_0)} = \Exv{\psi_0} \ge \Exv{\psi_T}.
  \]
  Since $1 - \tau_0 < 1$, it follows that $\log(1 - \tau_0) < 0$.  Therefore,
  \[
    0
    \ge
    \Exv{\psi_T}
    =
    \Exv{\log(1 - \tau_T)}
    +
    \frac{1}{2} \eta \Delta \Exv{T}.
  \]
  Applying Lemma \ref{lemmaTauNoSuccess},
  \[
    1 - \tau_T \ge \gamma n^{-1} p^{-2} q \epsilon,
  \]
  and so
  \[
    0
    \ge
    \log(\gamma n^{-1} p^{-2} q \epsilon)
    +
    \frac{1}{2} \eta \Delta \Exv{T}.
  \]
  Solving for the expected value of the stopping time,
  \[
    \Exv{T}
    \le
    \frac{2}{\eta \Delta \delta}
    \log\left( \frac{n p^2}{\gamma q \epsilon} \right).
  \]
  Finally, substituting $\eta$ in terms of $\gamma$ results in
  \[
    \Exv{T}
    \le
    \frac{4 n \sigma_a^2 p^2 (p + \epsilon)}{\Delta^2 \gamma \epsilon}
    \log\left( \frac{n p^2}{\gamma q \epsilon} \right),
  \]
  as desired.
\end{proof}

Finally, we prove Theorem \ref{thmAlecton}.

\begin{proof}[Proof of angular part of Theorem \ref{thmAlecton}]
  First, we notice that the total failure event up to time $t$ can be written
  as
  \[
    F_t = f_T \cup \left\{ T > t \right\}.
  \]
  That is, total failure up to time $t$ occurs if either failure happens before
  success (event $f_T$), or neither success nor failure happen before $t$.
  By the union bound,
  \[
    F_t \le \Prob{f_T} + \Prob{T > t}.
  \]
  Applying Markov's inequality,
  \[
    \Prob{F_t} \le \Prob{f_T} + \frac{1}{t} \Exv{T}.
  \]
  Finally, applying Lemmas \ref{lemmaFailureProb} and \ref{lemmaStoppingTime} produces
  \[
    \Prob{F_t}
    \le
    Z_p(\gamma)
    +
    \frac{4 n \sigma_a^2 p^2 (p + \epsilon)}{\Delta^2 \gamma \epsilon t}
    \log\left( \frac{n p^2}{\gamma q \epsilon} \right).
  \]
  This is the desired expression.
\end{proof}

\begin{proof}[Proof of radial part of Theorem \ref{thmAlecton}]
  Recall that in Alecton, $\bar R$ is defined as
  \[
    \bar R
    =
    \frac{1}{L} \sum_{l=0}^{L-1} \hat Y^T \tilde A_l \hat Y.
  \]
  Now, computing the expected distance to the mean,
  \begin{align*}
    &\Exv{\normf{\bar R - \hat Y^T A \hat Y}^2} \\
    &=
    \Exv{
      \normf{
        \frac{1}{L} \sum_{l=0}^{L-1} \hat Y^T \tilde A_l \hat Y
        -
        \hat Y^T A \hat Y
      }^2
    } \\
    &=
    \Exv{
      \normf{
        \frac{1}{L} \sum_{l=0}^{L-1} \hat Y^T (\tilde A_l - A) \hat Y
      }^2
    } \\
    &=
    \frac{1}{L^2}
    \Exv{
      \sum_{k=0}^{L-1} \sum_{l=0}^{L-1}
      \trace{ \hat Y^T (\tilde A_k - A)^T \hat Y \hat Y^T (\tilde A_l - A) \hat Y }
    }
  \end{align*}
  Since $\Exv{\tilde A} = A$, and the $\tilde A_l$ are independently sampled,
  the summand here will be zero unless $k = l$.  Therefore,
  \begin{align*}
    &\Exv{\normf{\bar R - \hat Y^T A \hat Y}^2} \\
    &=
    \frac{1}{L^2}
    \sum_{l=0}^{L-1}  
    \Exv{\trace{\hat Y^T (\tilde A_l - A)^T \hat Y \hat Y^T (\tilde A_l - A) \hat Y}} \\
    &=
    \frac{1}{L}
    \Exv{\trace{\hat Y^T (\tilde A - A)^T \hat Y \hat Y^T (\tilde A - A) \hat Y}} \\
    &\le
    \frac{1}{L}
    \Exv{\trace{\hat Y^T \tilde A^T \hat Y \hat Y^T \tilde A \hat Y}}.
  \end{align*}
  Applying the Alecton variance condition, and recalling that
  $\trace{\hat Y \hat Y^T} = p$, results in
  \[
    \Exv{\normf{\bar R - \hat Y^T A \hat Y}^2}
    \le
    \frac{p^2 \sigma_r^2}{L}.
  \]
  We can now apply Markov's inequality to this expression. 
  This results in, for any constant $\psi > 0$,
  \begin{dmath*}
    \Prob{
      \normf{\bar R - \hat Y^T A \hat Y}^2
      \ge 
      \psi
    }
    \le
    \frac{p^2 \sigma_r^2}{L \psi},
  \end{dmath*}
  which is the desired result.
\end{proof}

\section{Proofs of Lemmas}
\label{ssLemmaProofs}
First, we prove the lemmas used above to demonstrate the general result.

\begin{proof}[Proof of quadratic rational lower bound lemma (Lemma \ref{lemma-quadlb})]
  Expanding the product results in
  \begin{dmath*}
    \left( 1 + bx + cx^2 \right)
    \left( 1 + (2a - b) x - c x^2 \right)
    =
    1 + ((2a - b) + b)x + (c - c + (2a - b) b)x^2 + ((2a - b)c - bc)x^3 - c^2 x^4
    =
    1 + 2 a x + (2 ab - b^2) x^2 + 2 (a - b) c x^3 - c^2 x^4
    =
    1 + 2 a x + a^2 x^2 - (a^2 - 2ab + b^2) x^2 + 2 (a - b) c x^3 - c^2 x^4
    =
    1 + 2 a x + a^2 x^2 - x^2 \left((a - b)^2 - 2 (a - b) c x + c^2 x^2\right)
    =
    (1 + a x)^2 - x^2 ((a - b) - cx)^2
    \le 
      (1 + a x)^2.
  \end{dmath*}
  Dividing both sides by $1 + bx + cx^2$ (which we can do since this is assumed to be positive)
  reconstructs the desired identity.
\end{proof}

\begin{proof}[Proof of Lemma \ref{lemmaTau0}]
  We first note that, by the symmetry of the multivariate Gaussian
  distribution, initializing $Y_0$ uniformly at random such that
  $Y_0^T Y_0 = I$ is equivalent
  to initializing the entries of $Y_0$ as
  independent standard normal random variables, for the purposes
  of computing $\tau_0$.  Under
  this initialization strategy, $\Exv{\tau_0}$ is
  \begin{align*}
    \Exv{\tau_0}
    &=
    \Exv{\frac{\Det{Y_0^T U Y_0}}{\Det{Y_0^T W Y_0}}} \\
    &=
    \Exv{
      \frac{
        \Det{Y_0^T U Y_0}
      }{
        \Det{\gamma n^{-1} p^{-2} q Y_0^T (I - U) Y_0 + Y_0^T U Y_0}
      }
    }.
  \end{align*}
  Now, let $X \in \R^{q \times p}$ be the component of $Y_0$ that
  is in the column space
  of $U$, and let $Z \in \R^{(n-q) \times p}$ be the component
  of $Y_0$ in the null space of $U$.  Then,
  \[
    \Exv{\tau_0}
    =
    \Exv{\frac{\Det{X^T X}}{\Det{\gamma n^{-1} p^{-2} q Z^T Z + X^T X}}}.
  \]
  Since $X$ and $Z$ are selected orthogonally from a Gaussian
  random matrix, they must be independent, so we can take their
  expected values independently.  Taking the
  expected value first with respect to $Z$, we notice that $\Det{V}^{-1}$ is a 
  convex function in $V$, and so by Jensen's inequality,
  \begin{align*}
    \Exv{\tau_0}
    &\ge
    \Exv{\frac{\Det{X^T X}}{\Det{\gamma n^{-1} p^{-2} q \Exv{Z^T Z} + X^T X}}} \\
    &\ge
    \Exv{\frac{\Det{X^T X}}{\Det{\gamma n^{-1} p^{-2} q (n - q) I + X^T X}}} \\
    &\ge
    \Exv{\frac{\Det{X^T X}}{\Det{\gamma p^{-2} q I + X^T X}}} \\
    &=
    \Exv{\Det{I + \gamma p^{-2} q (X^T X)^{-1}}^{-1}}.
  \end{align*}
  Now, let $V \in \R^{q \times p}$ be a random full-rank
  projection matrix, selected independently of $X$.  Then,
  \[
    \Exv{V V^T} = \frac{p}{q} I,
  \]
  and so
  \[
    \Exv{\tau_0}
    \ge
    \Exv{\Det{I + \gamma p^{-1} \Exvc{X^T V V^T X}{X}^{-1}}^{-1}}.
  \]
  Applying Jensen's inequality again,
  \[
    \Exv{\tau_0}
    \ge
    \Exv{\Exvc{\Det{I + \gamma p^{-1} \left(X^T V V^T X \right)^{-1}}^{-1}}{X}}.
  \]
  and by the law of total expectation,
  \[
    \Exv{\tau_0}
    \ge
    \Exv{\Det{I + \gamma p^{-1} \left(X^T V V^T X \right)^{-1}}^{-1}}.
  \]
  Now, since $V$ and $X$ were sampled independently, it follows that
  $V^T X$ is sampled as a standard normal random matrix in $\R^{p \times p}$.
  If we call this matrix $R$, then
  \begin{align*}
    \Exv{\tau_0}
    &\ge
    \Exv{\Det{I + \gamma p^{-1} \left(R^T R \right)^{-1}}^{-1}} \\
    &=
    1 - \frac{1}{2} Z_p(\gamma),
  \end{align*}
  as desired.
\end{proof}

\begin{lemma}
  \label{lemmaDetRank1Update}
  For any $B \in \R^{n \times n}$, any $Y \in \R^{n \times m}$, and any symmetric positive-
  semidefinite $Z \in \R^{n \times n}$, if either $B$ is rank-1 or $m = 1$, then
  \begin{align*}
    &\Det{Y^T (I + B)^T Z (I + B) Y} \\
    &\hspace{1em}=
    \Det{Y^T Z Y} \Big( \left(\trace{Y (Y^T Z Y)^{-1} Y^T Z B} + 1\right)^2 \\
    &\hspace{2em}+
    \trace{Y (Y^T Z Y)^{-1} Y^T B^T Z B} \\
    &\hspace{2em}-
    \trace{Z Y (Y^T Z Y)^{-1} Y^T Z B Y (Y^T Z Y)^{-1} Y^T B^T} \Big).
  \end{align*}
\end{lemma}
\begin{proof}
  We will prove this separately for each case.  First, if $m = 1$, then $Y$
  is a vector, and the desired expression simplifies to
  \begin{align*}
    &Y^T (I + B)^T Z (I + B) Y \\
    &\hspace{1em}=
    Y^T Z Y \left((Y^T Z Y)^{-1} Y^T Z B Y + 1\right)^2 \\
    &\hspace{2em}+
    \trace{Y^T B^T Z B Y} \\
    &\hspace{2em}-
    (Y^T Z Y)^{-1} (Y^T Z B Y)^2.
  \end{align*}
  Straightforward evaluation indicates that this expression holds in this case.

  Next, we consider the case where $B$ is rank-1.
  In this case, we can rewrite it as $B = u v^T$
  for vectors $u$ and $v$, such that $u^T Z u = 1$.  Then,
  \begin{align*}
    &\Det{Y^T (I + B)^T Z (I + B) Y} \\
    &\hspace{1em}=
    \Det{Y^T (I + u v^T)^T Z (I + u v^T) Y} \\
    &\hspace{1em}=
    \Det{
      Y^T Z Y
      +
      2 Y^T Z u v^T Y
      +
      Y^T v v^T Y
    }
  \end{align*}
  If we define $M = Y^T Z Y$ and
  \[
    W = \left[ \begin{array}{c c} Y^T Z u & Y^T v \end{array} \right],
  \]
  then
  \begin{align*}
    &\Det{Y^T (I + B)^T Z (I + B) Y} \\
    &\hspace{1em}=
    \Det{
      M
      +
      W
      \left[ \begin{array}{c c} 0 & 1 \\ 1 & 1 \end{array} \right]
      W^T
    }.
  \end{align*}
  Applying the matrix determinant lemma, and recalling that
  \[
    \left[ \begin{array}{c c} 0 & 1 \\ 1 & 1 \end{array} \right]^{-1}
    =
    \left[ \begin{array}{c c} -1 & 1 \\ 1 & 0 \end{array} \right]
  \]
  and
  \[
    \left| \begin{array}{c c} 0 & 1 \\ 1 & 1 \end{array} \right|
    =
    -1,
  \]
  we produce
  \begin{align*}
    &- \det{M}^{-1} \Det{Y^T (I + B)^T Z (I + B) Y} \\
    &\hspace{1em}=
    -\Det{
      \left[ \begin{array}{c c} -1 & 1 \\ 1 & 0 \end{array} \right]
      +
      W^T M^{-1} W
    } \\
    &\hspace{1em}=
    \left| \begin{array}{c c} 
      u^T Z Y M^{-1} Y^T Z u - 1 & v^T Y M^{-1} Y^T Z u + 1 \\
      v^T Y M^{-1} Y^T Z u + 1 & v^T Y M^{-1} Y^T v
    \end{array} \right| \\
    &\hspace{1em}=
    \left(u^T Z Y M^{-1} Y^T Z u - 1\right)
    \left(v^T Y M^{-1} Y^T v \right) \\
    &\hspace{2em}-
    \left(v^T Y M^{-1} Y^T Z u + 1\right)^2 \\
    &\hspace{1em}=
    u^T Z Y M^{-1} Y^T Z u v^T Y M^{-1} Y^T v \\
    &\hspace{2em}-
    v^T Y M^{-1} Y^T v u^T Z u \\
    &\hspace{2em}-
    \left(v^T Y M^{-1} Y^T Z u + 1\right)^2.
  \end{align*}
  Rewriting this in terms of the matrix $B = u v^T$,
  \begin{align*}
    &- \det{M}^{-1} \Det{Y^T (I + B)^T Z (I + B) Y} \\
    &\hspace{1em}=
    \trace{Z Y M^{-1} Y^T Z B Y M^{-1} Y^T B^T} \\
    &\hspace{2em}-
    \trace{Y M^{-1} Y^T B^T Z B} \\
    &\hspace{2em}-
    \left(\trace{Y M^{-1} Y^T Z B} + 1\right)^2.
  \end{align*}
  Substitution produces the desired result.
\end{proof}

\begin{proof}[Proof of Lemma \ref{lemmaDetRank1UpdateLUBound}]
  First, for the lower bound, we notice that
  \[
    Z Y (Y^T Z Y)^{-1} Y^T Z
    \preceq
    Z,
  \]
  since the interior of the left expression is a projection matrix.
  This lets us conclude that
  \begin{align*}
    &\trace{Y (Y^T Z Y)^{-1} Y^T B^T Z B} \\
    &\hspace{1em}\ge
    \trace{Z Y (Y^T Z Y)^{-1} Y^T Z B Y (Y^T Z Y)^{-1} Y^T B^T}.
  \end{align*}
  Appling this to the result of Lemma \ref{lemmaDetRank1Update}
  produces the desired lower bound.

  For the upper bound, recall that, by the Cauchy-Schwarz inequality, for any 
  rank-1 matrix $A$,
  \[
    \trace{A}^2 \le \trace{A^T A}.
  \]
  Since $B$ is rank-1, it follows that
  \begin{align*}
    &\trace{Y (Y^T Z Y)^{-1} Y^T Z B} \\
    &\hspace{1em}\le
    \trace{Z Y (Y^T Z Y)^{-1} Y^T Z B Y (Y^T Z Y)^{-1} Y^T B^T}.
  \end{align*}
  Appling this to the result of Lemma \ref{lemmaDetRank1Update}
  produces the desired upper bound.
\end{proof}

\begin{lemma}
  \label{lemma0XI}
  For any symmetric matrix $0 \preceq X \preceq I$,
  \[
    \trace{I - X} \ge 1 - \Det{X}.
  \]
\end{lemma}
\begin{proof}
  If $x_1, x_2, \ldots, x_p$ are the eigenvalues of $x$, then this statement is
  equivalent to
  \[
    \left(\sum_{i=1}^p (1 - x_i)\right) - \left(1 - \prod_{i=1}^p x_i\right) > 0.
  \]
  If we let $f(X)$ denote this expression, then
  \[
    \frac{\partial f}{\partial x_j}
    =
    -1 + \frac{1}{x_j} \prod_{i=1}^p x_i
    \le
    0.
  \]
  It follows that the minimum of $f$ is attained at $X = I$.  However, when $X = I$,
  $f(X) = 0$, and so $f > 0$, which proves the lemma.
\end{proof}

\begin{proof}[Proof of Lemma \ref{lemmaPhiTau}]
  From the definition of $\phi_k$, if we let $Z^2 = \left(Y_k^T W Y_k \right)^{-1}$
  for $Z$ positive semidefinite, then
  \begin{dmath*}
    \phi_k
    =
    \trace{I - Y_k^T U^T U Y_k \left(Y_k^T W Y_k \right)^{-1}}
    =
    \trace{I - Z Y_k^T U^T U Y_k Z}.
  \end{dmath*}
  Since $0 \preceq Z Y_k^T U^T U Y_k Z \preceq I$, we can apply Lemma \ref{lemma0XI},
  which produces
  \begin{dmath*}
    \phi_k
    \ge
    1 - \Det{Z Y_k^T U^T U Y_k Z}
    =
    1 - \frac{\Det{Y_k^T U^T U Y_k}}{\Det{Y_k^T W Y_k}}
    =
    1 - \tau_k,
  \end{dmath*}
  which is the desired expression.
\end{proof}

\begin{proof}[Proof of Lemma \ref{lemmaTauNoSuccess}]
  Since the success event does not occur, it follows that there exists a
  $z \in \R^p$ such that
  \[
    \frac{\norm{U Y_k z}^2}{\norm{Y_k z}^2} \le 1 - \epsilon.
  \]
  If we let
  \[
    \hat Y_k = Y_k \left( Y_k^T Y_k \right)^{-\frac{1}{2}},
  \]
  and define $\hat z$ as the unit vector such that
  \[
    \hat z \propto \left( Y_k^T Y_k \right)^{\frac{1}{2}} z,
  \]
  then we can rewrite this as
  \[
    \norm{U \hat Y_k \hat z}^2 \le 1 - \epsilon.
  \]
  It follows that $\hat Y_k^T U \hat Y_k$ has an eigenvalues less than $1 - \epsilon$.

  Now, expanding $\tau_k$,
  \begin{dmath*}
    \tau_k
    =
    \frac{\Det{Y_k^T U Y_k}}{\Det{Y_k^T W Y_k}}
    =
    \frac{\Det{\hat Y_k^T U \hat Y_k}}{\Det{\hat Y_k^T W \hat Y_k}}
    =
    \Det{
      (1 - \gamma n^{-1} p^{-2} q) I
      +
      \gamma n^{-1} p^{-2} q \left(\hat Y_k^T U \hat Y_k \right)^{-1}
    }^{-1}
  \end{dmath*}
  Since this is a matrix that has eigenvalues
  between $0$ and $1$, it follows that its determinant is less than
  each of its eigenvalues.  From the analysis above, we can
  bound one of the eigenvalues of this matrix.  Doing this results
  in
  \begin{dmath*}
    \tau_k
    \le
    \left(
      (1 - \gamma n^{-1} p^{-2} q)
      +
      \gamma n^{-1} p^{-2} q \left(1 - \epsilon \right)^{-1}
    \right)^{-1}
    =
    \frac{
      1 - \epsilon
    }{
      \gamma n^{-1} p^{-2} q
      + 
      (1 - \gamma n^{-1} p^{-2} q) (1 - \epsilon)
    }
    =
    1
    -
    \frac{
      \gamma n^{-1} p^{-2} q \epsilon
    }{
      \gamma n^{-1} p^{-2} q
      + 
      (1 - \gamma n^{-1} p^{-2} q) (1 - \epsilon)
    }
    \le
    1 - \gamma n^{-1} p^{-2} q \epsilon,
  \end{dmath*}
  as desired.
\end{proof}

\begin{lemma}
  \label{lemmaZ1sub}
  Let $x$ be a standard normal random variable, and $a \in \R$ a constant.  Then
  \[
    \Exv{\frac{a^2}{x^2 + a^2}}
    =
    \exp\left(\frac{a^2}{2}\right) \sqrt{\frac{\pi a^2}{2}} 
      \mathrm{erfc}\left( \sqrt{\frac{a^2}{2}} \right).
  \]
\end{lemma}
\begin{proof}
  By the definition of expected value, since $x$ is normally distributed,
  \[
    \Exv{\frac{a^2}{x^2 + a^2}}
    =
    \int_{-\infty}^{\infty} \left( \frac{a^2}{x^2 + a^2} \right)
      \left( \frac{1}{\sqrt{2 \pi}} \exp\left( -\frac{x^2}{2} \right) \right) dx.
  \]
  If we let $\F$ denote the fourier transform, then
  \[
    \F\left[ \frac{a}{x^2 + a^2} \right]
    =
      \sqrt{2 \pi} \exp\left( - a \Abs{\omega} \right).
  \]
  Furthermore, since the Gaussian functions are eigenfunctions of the 
  Fourier transform, we know that
  \[
    \F\left[ \frac{1}{\sqrt{2 \pi}} \exp\left( -\frac{x^2}{2} \right) \right]
    =
       \frac{1}{\sqrt{2 \pi}} \exp\left( -\frac{\omega^2}{2} \right).
  \]
  And so, by Parseval's theorem,
  \begin{dmath*}
    \Exv{\frac{1}{x^2 + 1}}
    =
    a \int_{-\infty}^{\infty} \F\left[ \frac{a}{x^2 + a^2} \right]
      \F\left[ \frac{1}{\sqrt{2 \pi}} \exp\left( -\frac{x^2}{2} \right) \right] d\omega
    =
    a \int_{-\infty}^{\infty} 
      \sqrt{2 \pi} \exp\left( - a \Abs{\omega} \right)
      \left( \frac{1}{\sqrt{2 \pi}} \exp\left( -\frac{\omega^2}{2} \right) \right) d\omega
    =
    a \int_{0}^{\infty} 
      \exp\left(  - a \omega - \frac{\omega^2}{2} \right) d\omega
    =
    a \exp\left(\frac{a^2}{2}\right)
    \int_{0}^{\infty} 
      \exp\left( -\frac{a^2}{2} - a \omega - \frac{\omega^2}{2} \right) d\omega.
  \end{dmath*}
  Letting $u = \frac{\omega + a}{\sqrt{2}}$ and $d\omega = \sqrt{2} du$, so
  \begin{dmath*}
    \Exv{\frac{1}{x^2 + 1}}
    =
    a \exp\left(\frac{a^2}{2}\right)
    \int_{\frac{a}{\sqrt{2}}}^{\infty} 
      \exp\left( -u^2 \right) \sqrt{2} du
    =
    \exp\left(\frac{a^2}{2}\right) \sqrt{\frac{\pi a^2}{2}} 
      \mathrm{erfc}\left( \sqrt{\frac{a^2}{2}} \right),
  \end{dmath*}
  as desired.
\end{proof}

\begin{proof}[Proof of Lemma \ref{lemmaZ1}]
  We start by stating the definition of $Z_1(\gamma)$.  For some Gaussian
  random matrix $R \in \R^{1 \times 1}$,
  \[
    Z_1(\gamma)
    =
    2 \left(1 - \Exv{\Det{I + \gamma (R^T R)^{-1}}^{-1}}\right).
  \]
  Since $R$ is a scalar, this reduces to
  \begin{align*}
    Z_1(\gamma)
    &=
    2 \left(1 - \Exv{\left(1 + \gamma R^{-2} \right)^{-1}}\right) \\
    &=
    \Exv{ 2 \left(1 - \frac{1}{1 + \gamma R^{-2}} \right) } \\
    &=
    \Exv{ 2 \frac{\gamma R^{-2}}{1 + \gamma R^{-2}} } \\
    &=
    2 \Exv{ \left(\frac{\gamma}{R^2 + \gamma} \right) }.
  \end{align*}
  Applying Lemma \ref{lemmaZ1sub},
  \begin{align*}
    Z_1(\gamma)
    &=
    2 \exp\left(\frac{\gamma}{2}\right) \sqrt{\frac{\pi \gamma}{2}} 
      \mathrm{erfc}\left( \sqrt{\frac{\gamma}{2}} \right) \\
    &=
    \sqrt{2 \pi \gamma} \exp\left(\frac{\gamma}{2}\right)
    \mathrm{erfc}\left( \sqrt{\frac{\gamma}{2}} \right).
  \end{align*}
  This is the desired expression.  Furthermore, since for all $x$,
  \[
    \mathrm{erfc}\left( \sqrt{x} \right)
    \le
    \exp\left(x\right),
  \]
  we can also produce the desired upper bound on $Z_1$,
  \[
    Z_1 \le \sqrt{2 \pi \gamma}.
  \]
\end{proof}

\subsection{Proofs of Alecton Variance Condition Lemmas}

Next, we prove the Alecton Variance Conditions lemmas for
the distributions mentioned in the body of the paper.

\subsubsection{Entrywise Sampling}

To analyze the entrywise sampling case, we need some lemmas that
makes the incoherence condition more accessible.
\begin{lemma}
  \label{lemmaIncoherenceCommute}
  If matrix $A$ is symmetric and incoherent with parameter $\mu$, and $B$ is a
  symmetric matrix that commutes with $A$, then $B$ is incoherent with parameter
  $\mu$.
\end{lemma}
\begin{proof}
  Since $A$ and $B$ commute, they must have the same eigenvectors.  Therefore,
  the set of eigenvectors that shows that $A$ is incoherent with parameter $\mu$
  will also show that $B$ has the same property.
\end{proof}

\begin{lemma}
  \label{lemmaIncoherenceBound}
  If matrix $A$ is symmetric and incoherent with parameter $\mu$, and $e_i$ is a
  standard basis element, then
  \[
    e_i^T A e_i \le \frac{\mu^2}{n} \trace{A}.
  \]
\end{lemma}
\begin{proof}
  Let $u_1, u_2, \ldots, u_n$ be the eigenvectors guaranteed by the incoherence of $A$,
  and let $\lambda_1, \ldots, \lambda_n$ be the corresponding eigenvalues.  Then,
  \begin{align*}
    e_i^T A e_i
    &=
    e_i^T \left(\sum_{j=1}^n u_j \lambda_j u_j^T \right) e_i \\
    &=
    \sum_{j=1}^n u_j \lambda_j (e_i^T u_j)^2.
  \end{align*}
  Applying the definition of incoherence,
  \begin{align*}
    e_i^T A e_i
    &\le
    \sum_{j=1}^n u_j \lambda_j \left(\frac{\mu}{\sqrt{n}} \right)^2
    =
    \frac{\mu^2}{n} \trace{A},
  \end{align*}
  as desired.
\end{proof}

\begin{proof}[Proof of the $\sigma_a$ bound part of Lemma \ref{lemmaEntrywiseAVC}]
  We recall that the entrywise samples are of the form
  \[
    \tilde A
    =
    n^2 u u^T A v v^T,
  \]
  where $u$ and $v$ are independently, uniformly chosen standard basis elements.
  We further recall that $\Exv{u u^T} = \Exv{v v^T} = n^{-1} I$.
  Now, evaluating the desired quantity,
  \[
    \Exv{y^T \tilde A^T W \tilde A y}
    =
    n^4 \Exv{y^T v v^T A u u^T W u u^T A v v^T y}.
  \]
  Since $W$ commutes with $A$, by Lemmas \ref{lemmaIncoherenceCommute} and
  \ref{lemmaIncoherenceBound}, $u^T W u \le \mu^2 n^{-1} \trace{W}$.  Therefore,
  \begin{align*}
    \Exv{y^T \tilde A^T W \tilde A y}
    &\le
    \mu^2 n^3 \trace{W} \Exv{y^T v v^T A u u^T A v v^T y} \\
    &=
    \mu^2 n^2 \trace{W} \Exv{y^T v v^T A^2 v v^T y}.
  \end{align*}
  Since $A^2$ commutes with $A$, the same logic shows that $v^T A^2 v \le \mu^2 n^{-1} \trace{A^2}$,
  and so,
  \begin{align*}
    \Exv{y^T \tilde A^T W \tilde A y}
    &\le
    \mu^4 n \trace{W} \trace{A^2} \Exv{y^T v v^T y} \\
    &=
    \mu^4 \trace{W} \normf{A}^2 \norm{y}^2.
  \end{align*}
  So it suffices to choose $\sigma_a^2 = \mu^4 \normf{A}^2$, as desired.
\end{proof}

\begin{proof}[Proof of the $\sigma_r$ bound part of Lemma \ref{lemmaEntrywiseAVC}]
  Evaluating the desired quantity,
  \begin{align*}
    \Exv{\left(y^T \tilde A y\right)^2}
    &=
    n^4 \Exv{\left(y^T u u^T A v v^T y\right)^2} \\
    &=
    n^4 \Exv{(u^T y)^2 (v^T y)^2 (u^T A v)^2}.
  \end{align*}
  By the Cauchy–Schwarz inequality,
  \[
    (u^T A v)^2 \le (u^T A u) (v^T A v),
  \]
  and by Lemma \ref{lemmaIncoherenceBound}, $u^T A u \le \mu^2 n^{-1} \trace{A}$, 
  and so
  \[
    (u^T A v)^2 \le \mu^4 n^{-2} \trace{A}^2.
  \]
  Therefore,
  \begin{align*}
    \Exv{\left(y^T \tilde A y\right)^2}
    &\le
    \mu^4 n^2 \trace{A}^2 \Exv{(u^T y)^2 (v^T y)^2} \\
    &=
    \mu^4 \trace{A}^2 \norm{y}^4.
  \end{align*}
  So it suffices to choose $\sigma_r^2 = \mu^4 \trace{A}^2$, as desired.
\end{proof}

\subsubsection{Rectangular Entrywise Sampling}

\begin{proof}[Proof of Lemma \ref{lemmaEntrywiseRectAVC}]
  We recall that the rectangular entrywise samples are of the form
  \[
    \tilde A
    =
    m n M_{ij} (e_i e_{m+j}^T + e_{m+j} e_i^T),
  \]
  where $i \in {1,\ldots,m}$ and $j \in {1,\ldots,n}$ are chosen uniformly and independently.
  Now, for any $y$ and $z$ in $\R^{m + n}$,
  \[
    \Exv{(z^T \tilde A y)^2}
    =
    m^2 n^2 \Exv{M_{ij}^2 (z^T (e_i e_{m+j}^T + e_{m+j} e_i^T) y)^2}.
  \]
  Applying the entry bound,
  \[
    \Exv{(z^T \tilde A y)^2}
    \le
    \xi m n \normf{M}^2 \Exv{(z^T e_i e_{m+j}^T y + z^T e_{m+j} e_i^T y)^2}.
  \]
  Now, since $(x + y)^2 \le 2 (x^2 + y^2)$, if we let $P$ be the projection matrix onto the
  first $m$ basis vectors, then $\Exv{e_i e_i^T} = m^{-1} P$ and
  $\Exv{e_{m+j} e_{m+j}^T} = n^{-1} (I - P)$, and so,
  \begin{align*}
    &\Exv{(z^T \tilde A y)^2} \\
    &\le
    2 \xi m n \normf{M}^2 \Exv{(z^T e_i)^2 (e_{m+j}^T y)^2 + (z^T e_{m+j})^2 (e_i^T y)^2} \\
    &=
    2 \xi \normf{M}^2 \left(\norm{P z}^2 \norm{(I - P) y}^2 + \norm{(I - P) z}^2 \norm{P y}^2 \right) \\
    &\le
    2 \xi \normf{M}^2 \norm{y}^2 \norm{z}^2.
  \end{align*}
  Since this is true for any $y$ and $z$, it is true in particular for $z$ being an eigenvector
  of $A$.  Therefore, it suffices to pick $\sigma_a^2 = 2 \xi \normf{M}^2$.  Similarly,
  it is true in particular for $z = y$, and therefore it suffices to pick
  $\sigma_r^2 = 2 \xi \normf{M}^2$.  This proves the lemma.
\end{proof}

\subsubsection{Trace Sampling}

In order to prove our second moment lemma for the trace sampling case, we must
first derive some lemmas about the way this distribution behaves. 

\begin{lemma}[Sphere Component Fourth Moment]
  \label{lemmaSC4M}
  If $n > 50$, and $v \in \R^n$ is sampled uniformly from the unit sphere,
  then for any unit vector $y \in \R^n$,
  \[
    \Exv{\left(y^T v\right)^4} \le \frac{4}{n^2}.
  \]
\end{lemma}
\begin{proof}
  Let $x$ be sampled from the standard normal distribution in $\R^n$.  Then, by
  radial symmetry,
  \[
    \Exv{\left(y^T v\right)^4}
    =
    \Exv{\frac{\left(y^T x\right)^4}{\norm{x}^4}}.
  \]
  If we let $u$ denote $y^T x$, and $z$ denote the components of $x$ orthogonal
  to $y$, then $\norm{x}^2 = u^2 + \norm{z}^2$.  Furthermore, by the
  properties of the normal distribution, $u$ and $z$ are independent.  Therefore,
  \begin{align*}
    \Exv{\left(y^T v\right)^4}
    &=
    \Exv{u^4 \left(u^2 + \norm{z}^2 \right)^{-2}} \\
    &\le
    \Exv{u^4 \left(\norm{z}^2\right)^{-2}} \\
    &=
    \Exv{u^4} \Exv{\norm{z}^{-4}}.
  \end{align*}
  Now, $\Exv{u^4}$ is the fourth moment of the normal distribution, which is known
  to be $3$.  Furthermore, $\Exv{\norm{z}^{-4}}$ is the second moment of an
  inverse-chi-squared distribution with parameter $n - 1$, which is also
  a known result.  Substituting these in,
  \begin{align*}
    \Exv{\left(y^T v\right)^4}
    &\le
    3 \left(\left(n - 3\right)^{-2} 
      + 2 \left(n - 3\right)^{-2} \left(n - 5\right)^{-1} \right) \\
    &=
    3 \left(n - 3\right)^{-2} \left(1
      + 2 \left(n - 5\right)^{-1} \right).
  \end{align*}
  This quantity has the asymptotic properties we want.  In particular, applying
  the constraint that $n > 50$,
  \[
    \Exv{\left(y^T v\right)^4}
    \le
    \frac{4}{n^2}.
  \]
  This is the desired result.
\end{proof}

\begin{lemma}[Sphere Component Fourth Moment Matrix]
  \label{lemmaSC4MM}
  If $n > 50$, and $v \in \R^n$ is sampled uniformly from the unit sphere,
  then for any positive semidefinite matrix $W$,
  \[
    \Exv{v v^T W v v^T} \preceq 4 n^{-2} \trace{W} I.
  \]
\end{lemma}
\begin{proof}
  Let
  \[
    W = \sum_{i=1}^n \lambda_i w_i w_i^T
  \]
  be the eigendecomposition of $W$.  Then for any unit vector $z$,
  \begin{align*}
    z^T \Exv{v v^T W v v^T} z
    &=
    \Exv{z^T v v^T \left(\sum_{i=1}^n \lambda_i w_i w_i^T \right) v v^T z} \\
    &=
    \sum_{i=1}^n \lambda_i \Exv{\left(z^T v\right)^2 \left(w_i^T v\right)^2}.
  \end{align*}
  By the Cauchy-Schwarz inequality applied to the expectation,
  \begin{align*}
    \Exv{\left(z^T v\right)^2 \left(w_i^T v\right)^2}
    &\le
    \sqrt{\Exv{\left(z^T v\right)^4} \Exv{\left(w_i^T v\right)^2}} \\
    &=
    \Exv{(z^T v)^4}.
  \end{align*}
  By Lemma \ref{lemmaSC4M}, $\Exv{(z^T v)^4} \le 4 n^{-2}$,
  and so
  \begin{align*}
    z^T \Exv{v v^T W v v^T} z
    &\le
    \sum_{i=1}^n \lambda_i (4 n^{-2})
    =
    4 n^{-2} \trace{W}.
  \end{align*}
  Since this is true for any unit vector $z$, by the definition of the
  positive semidefinite relation,
  \[
    \Exv{v v^T W v v^T}
    \preceq
    4 n^{-2} \trace{W} I,
  \]
  as desired.
\end{proof}

Now, we prove the AVC lemma for this distribution.

\begin{proof}[Proof of $\sigma_a$ bound part of Lemma \ref{lemmaTraceAVC}]
  Evaluating the expression we want to bound,
  \[
    \Exv{y^T \tilde A^T W \tilde A y}
    =
    n^4 \Exv{y^T v v^T A u u^T W u u^T A v v^T y}.
  \]
  Applying Lemma \ref{lemmaSC4MM},
  \begin{align*}
    \Exv{y^T \tilde A^T W \tilde A y}
    &\le
    n^4 \Exv{y^T v v^T A \left(4 n^{-2} \trace{W} I \right) A v v^T y} \\
    &=
    4 n^2 \trace{W} \Exv{y^T v v^T A^2 v v^T y}.
  \end{align*}
  Again applying Lemma \ref{lemmaSC4MM},
  \begin{align*}
    \Exv{y^T \tilde A^T W \tilde A y}
    &\le
    4 n^2 \trace{W} y^T \left(4 n^{-2} \trace{A^2} I \right) y \\
    &=
    16 \normf{A}^2 \trace{W} \norm{y}^2.
  \end{align*}
  So it suffices to pick $\sigma_a^2 = 16 \normf{A}^2$, as desired.
\end{proof}

\begin{proof}[Proof of $\sigma_r$ bound part of Lemma \ref{lemmaTraceAVC}]
  Evaluating the expression we want to bound,
  \begin{align*}
    \Exv{\left(y \tilde A y \right)^2}
    &=
    n^4 \Exv{\left(y v v^T A w w^T y \right)^2} \\
    &=
    n^4 \Exv{\trace{A v v^T y y^T v v^T A w w^T y y^T w w^T}} \\
    &=
    n^4 \trace{A \Exv{v v^T y y^T v v^T} A \Exv{w w^T y y^T w w^T}}.
  \end{align*}
  Applying Lemma \ref{lemmaSC4MM} to this results in
  \begin{align*}
    &\Exv{\left(y \tilde A y \right)^2} \\
    &\hspace{1em}\le
    n^4 \trace{A \left(4 n^{-2} \trace{y y^T} I \right) A \left(4 n^{-2} \trace{y y^T} I \right)} \\
    &\hspace{1em}=
    16 \normf{A}^2 \norm{y}^4.
  \end{align*}
  So it suffices to pick $\sigma_r^2 = 16 \normf{A}^2$, as desired.
\end{proof}

\subsubsection{Subspace Sampling}

Recall that, in subspace sampling, our samples are of the form
\[
  \tilde A = r n^2 m^{-2} Q v v^T R,
\]
where $Q$ and $R$ are independent projection matrices that select $m$ entries
uniformly at random, and $v$ is uniformly and independently selected from the column space of
$A$.  Using this, we first prove some lemmas, then prove our bounds.

\begin{lemma}
  \label{lemmaIncoherenceMatrix}
  If $Q$ is a projection matrix that projects onto a
  subspace spanned by $m$ random standard basis vectors, and
  $v$ is a member of a subspace that is incoherent with parameter $\mu$, then
  for any vector $x$,
  \[
    (x^T Q v)^2
    \le
    (\mu m r + m^2) n^{-2} \norm{x}^2 \norm{v}^2.
  \]
  As a corollary, for any symmetric matrix $W \succeq 0$,
  \[
    v^T Q W Q v
    \le
    (\mu m r + m^2) n^{-2} \trace{W} \norm{v}^2.
  \]
\end{lemma}
\begin{proof}
  Let $\lambda_i$ be $1$ in the event that $e_i$ is in the column space of $Q$, and $0$
  otherwise.  Then an eigendecomposition of $Q$ is
  \[
    Q = \sum_{i=1}^n \lambda_i e_i e_i^T.
  \]
  Therefore,
  \begin{align*}
    (x^T Q v)^2
    &=
    \left(\sum_{i=1}^n \lambda_i x^T e_i e_i^T v \right)^2 \\
    &=
    \sum_{i=1}^n \sum_{j=1}^n \lambda_i \lambda_j x_i x_j v_i v_j.
  \end{align*}
  Taking the expected value, and noting that $\lambda_i$ and $\lambda_j$ are independent,
  and have expected value $\Exv{\lambda_i} = m n^{-1}$,
  \begin{align*}
    \Exv{(x^T Q v)^2}
    &=
    m^2 n^{-2} \sum_{i=1}^n \sum_{j=1}^n x_i x_j v_i v_j \\
    &\hspace{1em} +
    m n^{-1} (1 - m n^{-1}) \sum_{i=1}^n x_i^2 v_i^2.
  \end{align*}
  Since $v$ is part of a subspace that is incoherent,
  \begin{align*}
    \Exv{(x^T Q v)^2}
    &\le
    m^2 n^{-2} \sum_{i=1}^n \sum_{j=1}^n x_i x_j v_i v_j \\
    &\hspace{1em} +
    \mu m r n^{-2} (1 - m n^{-1}) \norm{v}^2 \sum_{i=1}^n x_i^2 \\
    &=
    m^2 n^{-2} (x^T v)^2 \\
    &\hspace{1em} +
    \mu m r n^{-2} \norm{x}^2 \norm{v}^2 \\
    &\le
    (\mu m r + m^2) n^{-2} \norm{x}^2 \norm{v}^2,
  \end{align*}
  as desired.
\end{proof}

\begin{proof}[Proof of $\sigma_a$ bound part of Lemma \ref{lemmaSubspaceAVC}]
  Evaluating the expression we want to bound,
  \begin{align*}
    &\Exv{y^T \tilde A^T W \tilde A y} \\
    &\hspace{1em}=
    r^2 n^4 m^{-4} \Exv{y^T R v v^T Q W Q v v^T R y} \\
    &\hspace{1em}=
    r^2 n^4 m^{-4} \Exv{\Exv{v^T R y y^T R v} \Exv{v^T Q W Q v}}.
  \end{align*}
  Applying Lemma \ref{lemmaIncoherenceMatrix},
  \begin{align*}
    \Exv{y^T \tilde A^T W \tilde A y}
    &\le
    r^2 m^{-4} (\mu m r + m^2)^2 \trace{W} \norm{y}^2 \\
    &=
    r^2 (1 + \mu r m^{-1})^2 \trace{W} \norm{y}^2.
  \end{align*}
  So, we can choose $\sigma_a^2 = r^2 (1 + \mu r m^{-1})^2$, as desired.
\end{proof}

\begin{proof}[Proof of $\sigma_r$ bound part of Lemma \ref{lemmaSubspaceAVC}]
  Evaluating the expression we want to bound,
  \begin{align*}
    \Exv{(y^T \tilde A y)^2}
    &=
    r^2 n^4 m^{-4} \Exv{(y^T Q v v^T R y)^2} \\
    &=
    r^2 n^4 m^{-4} \Exv{\Exv{(y^T Q v)^2} \Exv{(y^T R v)^2}}.
  \end{align*}
  Applying Lemma \ref{lemmaIncoherenceMatrix},
  \begin{align*}
    \Exv{(y^T \tilde A y)^2}
    &\le
    r^2 m^{-4} (\mu m r + m^2)^2 \norm{y}^4 \\
    &=
    r^2 (1 + \mu r m^{-1})^2 \norm{y}^4.
  \end{align*}
  So, we can choose $\sigma_r^2 = r^2 (1 + \mu r m^{-1})^2$, as desired.
\end{proof}

\section{Lower Bound on Alecton Rate}
\label{ssLowerBoundSGD}
In this section, we prove a rough lower bound on the rate of convergence of
an Alecton-like algorithm for bounded sampling distributions.  Specifically,
we analyze the case where, rather than choosing a constant $\eta$, we allow
the step size to vary at each timestep.
Our result shows that we can't hope for a better
step size rule that improves the convergence rate of Alecton to, for example, a
linear rate.

To show this lower bound, we assume we run Alecton with $p = 1$ for some sampling
distribution such that for all $\eta$ and all $y$, for some constant $C$,
\[
  \norm{y + \eta \tilde A y} \le (1 + \eta C) \norm{y}.
\]
Further assume that for some eigenvector $u$ (with eigenvalue $\lambda \ge 0$)
that is not global solution, the sample variance in the direction of $u$ satisfies
\[
  \Exv{\tilde A^T u u^T \tilde A} \ge \sigma^2 I.
\]
We now define $\rho_k$ to be
\[
  \rho_k = \frac{(u^T Y_k)^2}{\norm{Y_k}^2}.
\]
This quantity measures the error of the iterate at timestep $k$ in
the direction of $u$.
We will show that the expected value of $\rho_k$ can only decrease
with at best a $\Omega\left(\frac{1}{K + 1}\right)$ rate.

First, we require a lemma.
\begin{lemma}
  \label{lemmaLowerBoundQuadSGD}
  For any $a \ge 0$, $b \ge 0$, and $0 \le x \le 1$,
  \[
    a (1 - x)^2 + b x^2 \ge \frac{a b}{a + b}.
  \]
\end{lemma}
\begin{proof}
  Expanding the left side,
  \begin{align*}
    a (1 - x)^2 + b x^2
    &=
    a - 2 a x + (a + b) x^2 \\
    &=
    a - \frac{a^2}{a + b} + \frac{a^2}{a + b} - 2 a x + (a + b) x^2 \\
    &=
    \frac{ab}{a + b} + \frac{(a - (a + b) x)^2}{a + b} \\
    &\ge
    \frac{ab}{a + b},
  \end{align*}
  as desired.
\end{proof}

\begin{theorem}
  Under the above conditions, regardless of how we choose the step size in the
  Alecton algorithm, even if we are able to choose a different step size each
  iteration, the expected error will still satisfy
  \[
    \Exv{\rho_K} \ge \frac{\sigma^2}{\sigma^2 n + C^2 K}.
  \]
\end{theorem}
\begin{proof}
  Using the Alecton update rule with a time-varying step size $\eta_k$,
  \begin{align*}
    \rho_{k+1} 
    &=
    \frac{(u^T Y_k)^2}{\norm{Y_k}^2} \\
    &=
    \frac{(u^T Y_k + \eta_k u^T \tilde A_k Y_k)^2}{\norm{Y_k + \eta_k \tilde A_k Y_k}^2} \\
    &\ge
    \frac{(u^T Y_k + \eta_k u^T \tilde A_k Y_k)^2}{(1 + \eta_k C)^2 \norm{Y_k}^2}.
  \end{align*}
  Taking the expected value,
  \begin{align*}
    \Exv{\rho_{k+1}}
    &\ge
    \Exv{\frac{(u^T Y_k + \eta_k u^T \tilde A_k Y_k)^2}{(1 + \eta_k C)^2 \norm{Y_k}^2}} \\
    &\ge
    \Exv{\frac{(1 + 2 \eta_k \lambda) (u^T Y_k)^2 + \eta_k^2 \sigma^2 Y_k^T Y_k}{(1 + \eta_k C)^2 \norm{Y_k}^2}} \\
    &=
    \frac{1 + 2 \eta_k \lambda}{(1 + \eta_k C)^2} \Exv{\rho_k} + \frac{\eta_k^2 \sigma^2}{(1 + \eta_k C)^2} \\
    &\ge
    \frac{1}{(1 + \eta_k C)^2} \Exv{\rho_k} + \frac{\eta_k^2 \sigma^2}{(1 + \eta_k C)^2}
  \end{align*}
  Now, if we define $\zeta_k$ as
  \[
    \zeta_k = \frac{\eta_k C}{1 + \eta_k C},
  \]
  then
  \[
    \Exv{\rho_{k+1}}
    \ge
    (1 - \zeta_k)^2 \Exv{\rho_k} + \zeta_k^2 \sigma^2 C^{-2}.
  \]
  Applying Lemma \ref{lemmaLowerBoundQuadSGD},
  \[
    \Exv{\rho_{k+1}}
    \ge
    \frac{\sigma^2 C^{-2} \Exv{\rho_k}}{\Exv{\rho_k} + \sigma^2 C^{-2}}.
  \]
  Taking the inverse,
  \[
    \frac{1}{\Exv{\rho_{k+1}}}
    \le
    \frac{1}{\Exv{\rho_k}}
    +
    \frac{C^2}{\sigma^2}.
  \]
  Therefore, summing across steps,
  \[
    \frac{1}{\Exv{\rho_K}}
    \le
    \frac{1}{\Exv{\rho_0}}
    +
    \frac{C^2 K}{\sigma^2}.
  \]
  Since, by symmetry, $\Exv{\rho_0} = n^{-1}$, we have
  \[
    \frac{1}{\Exv{\rho_K}}
    \le
    n
    +
    \frac{C^2 K}{\sigma^2}.
  \]
  and taking the inverse again produces
  \[
    \Exv{\rho_K}
    \ge
    \frac{\sigma^2}{\sigma^2 n + C^2 K},
  \]
  which is the desired expression.
\end{proof}

\section{Handling Constraints}

Alecton can easily be adapted to solve the problem of finding a low-rank approximation to a
matrix under a \emph{spectahedral constraint}.  That is, we want to solve the problem
\[
  \begin{array}{ll}
    \mbox{minimize} & \normf{A - X}^2 \\
    \mbox{subject to} & X \in \R^{N \times N}, \trace{X} = 1, \\
    & \rank{X} \le 1, X \succeq 0.
  \end{array}
\]
This is equivalent to the decomposed problem
\[
  \begin{array}{ll}
    \mbox{minimize} & \norm{y}^4 - 2 y^T A y + \normf{A}^2 \\
    \mbox{subject to} & y \in \R^N, \norm{y}^2 = 1,
  \end{array}
\]
which is itself equivalent to:
\[
  \begin{array}{ll}
    \mbox{minimize} & 1 - 2 y^T A y + \normf{A}^2 \\
    \mbox{subject to} & y \in \R^N, \norm{y}^2 = 1.
  \end{array}
\]
This will have a minimum when $y = u_1$.  We can therefore solve the problem
using only the angular phase of Alecton, which recovers the vector $u_1$.
The same convergence analysis described above still applies.

For an example of a constrained problem that Alecton cannot handle, because it is
NP-hard, see the elliptope-constrained MAXCUT embedding in Appendix \ref{ssCounterexamples}.
This shows that constrained problems can't be solved efficiently by SGD algorithms in
all cases.

\section{Towards a Linear Rate}

In this section, we consider a special case of the matrix recovery problem: one in
which the samples we are given would allow us to exactly recover $A$.  That is,
for some linear operator $\Omega: \R^{n \times n} \rightarrow \R^s$, we are given
the value of $\Omega(A)$ as an input, and we know that the unique solution of the
optimization problem
\[
  \begin{array}{ll}
    \mbox{minimize} & \norm{\Omega(X - A)}^2 \\
    \mbox{subject to} & X \in \R^{n \times n}, \rank{X} \le p, X \succeq 0
  \end{array}
\]
is $X = A$.  Performing a rank-$p$ quadratic substitution on this problem
results in:
\[
  \begin{array}{ll}
    \mbox{minimize} & \norm{\Omega(Y Y^T - A)}^2 \\
    \mbox{subject to} & Y \in \R^{n \times p}
  \end{array}
\]
The specific case we will be looking at is where the operator
$\Omega$ satisfies the $p$-RIP constraint.
\begin{definition}[Restricted isometry property]
  A linear operator $\Omega: \R^{n \times n} \rightarrow \R^s$ satisfies
  $p$-RIP with constant $\delta$ if for all $X \in \R^{n \times n}$ of rank
  at most $p$,
  \[
    (1 - \delta) \normf{X}^2
    \le
    \norm{\Omega(X)}^2
    \le
    (1 + \delta) \normf{X}^2.
  \]
\end{definition}
This definition encodes the notion that $\Omega$ preserves the norm of 
low-rank matrices under its transformation.
We can prove a simple lemma that extends this to the inner product.
\begin{lemma}
  If $\Omega$ is $(p+q)$-RIP with parameter $\delta$, then for any symmetric
  matrices $X$ and $Y$ of rank at most $p$ and $q$ respectively,
  \[
    \Omega(X)^T \Omega(Y)
    \ge
    \trace{XY} - \delta \normf{X} \normf{Y}
  \]
\end{lemma}
\begin{proof}
  For any $a \in \R$, since $\Omega$ is linear,
  \begin{dmath*}
    \trace{\Omega(X) \Omega(Y)}
    =
    \frac{1}{4 a} \left(\norm{\Omega(X) + a \Omega(Y)}^2
      - \norm{\Omega(X) - a \Omega(Y)}^2 \right)
    =
    \frac{1}{4 a} \left(\norm{\Omega(X + a Y)}^2
      - \norm{\Omega(X - a Y)}^2 \right).
  \end{dmath*}
  Since $\rank{X - a Y} \le \rank{X} + \rank{Y} \le p + q$, we can apply our
  RIP inequalities, which produces
  \begin{dmath*}
    \trace{\Omega(X) \Omega(Y)}
    \ge
    \frac{1}{4 a} \left((1 - \delta) \normf{X + a Y}^2
      - (1 + \delta) \normf{X - a Y}^2 \right)
    \ge
    \frac{1}{4 a} \left(
      -2 \delta \normf{X}^2 + 4 a \trace{XY} -2 \delta a^2 \normf{Y}^2
    \right)
    =
    \trace{XY} - \delta \frac{\normf{X}^2 + a^2 \normf{Y}^2}{2 a}.
  \end{dmath*}
  Substituting $a = \frac{\normf{X}}{\normf{Y}}$ results in
  \[
    \trace{\Omega(X) \Omega(Y)}
    \ge
    \trace{XY} - \delta \normf{X} \normf{Y},
  \]
  as desired.
\end{proof}

Finally, we prove our main theorem that shows that the
quadratically transformed objective function is strongly convex in
a ball about the solution.  
\begin{theorem}
  If we define $f(Y)$ as the objective function of the above
  optimization problem, that is for $Y \in \R^{n \times p}$
  and $A \in \R^{n \times n}$ symmetric of rank no greater than
  $p$,
  \[
    f(Y) = \norm{\Omega(Y Y^T - A)}^2,
  \]
  and $\Omega$ is $3p$-RIP with parameter $\delta$, then for all
  $Y$, if we let $\lambda_p$ denote the smallest positive eigenvalue
  of $A$ then
  \[
    \nabla^2_V f(Y)
    \succeq
    2 \left(
      (1 - \delta) \lambda_p
      - (3 + \delta) \normf{Y Y^T - A} 
    \right) I.
  \]
\end{theorem}
\begin{proof}
  The directional derivative of $f$ along some direction $V$ will be,
  by the product rule,
  \[
    \nabla_V f(Y)
    =
    2 \Omega(Y Y^T - A)^T \Omega(Y V^T + V Y^T).
  \]
  The second derivative along this same direction will be
  \begin{dmath*}
    \nabla^2_V f(Y)
    =
    4 \Omega(Y Y^T - A)^T \Omega(V V^T)
    + 
    2 \Omega(Y V^T + V Y^T)^T \Omega(Y V^T + V Y^T)
    =
    4 \Omega(Y Y^T - A)^T \Omega(V V^T)
    + 
    2 \norm{\Omega(Y V^T + V Y^T)}^2.
  \end{dmath*}
  To this, we can apply the definition of RIP, and the corollary lemma,
  which results in
  \begin{dmath*}
    \nabla^2_V f(Y)
    \ge
    4 \trace{(Y Y^T - A)(U U^T)}
    - 4 \delta \normf{Y Y^T - A} \normf{U U^T}
    + 
    2 (1 - \delta) \normf{Y U^T + U Y^T}^2.
  \end{dmath*}
  By Cauchy-Schwarz,
  \begin{dmath*}
    \nabla^2_V f(Y)
    \ge
    - 4 \normf{Y Y^T - A} \trace{U U^T}
    - 4 \delta \normf{Y Y^T - A} \trace{U U^T}
    + 
    2 (1 - \delta) \lambda_{\text{min}}(Y^T Y) \trace{U U^T}
    =
    2 \left(
      (1 - \delta) \lambda_{\text{min}}(Y^T Y)
      - 2 (1 + \delta) \normf{Y Y^T - A} 
    \right) \trace{U U^T}.
  \end{dmath*}
  Now, since at the optimum, $\lambda_{\text{min}}(Y^T Y) = \lambda_p$,
  it follows that for general $Y$,
  \[
    \lambda_{\text{min}}(Y^T Y) \ge \lambda_p - \normf{Y Y^T - A}.
  \]
  Substituting this in to the previous expression,
  \begin{dmath*}
    \nabla^2_V f(Y)
    \ge
    2 \left(
      (1 - \delta) (\lambda_p - \normf{Y Y^T - A})
      - 2 (1 + \delta) \normf{Y Y^T - A} 
    \right) \trace{U U^T}
    =
    2 \left(
      (1 - \delta) \lambda_p
      - (3 + \delta) \normf{Y Y^T - A} 
    \right) \normf{U}^2.
  \end{dmath*}
  Since this is true for an arbitrary direction vector $U$, it follows that
  \[
    \nabla^2_V f(Y)
    \succeq
    2 \left(
      (1 - \delta) \lambda_p
      - (3 + \delta) \normf{Y Y^T - A} 
    \right) I,
  \]
  which is the desired result.
\end{proof}

This theorem shows that there is a region of size $O(1)$ (i.e. not dependent
on $n$) within which the above problem is strongly convex.  So, if
we start within this region, any standard convex descent method will converge
at a linear rate.  In particular, coordinate descent will do so.  Therefore,
we can imagine doing the following:
\begin{itemize}
  \item First, use Alecton to, with high probability, recover an
    estimate $Y$ that for which $\normf{Y Y^T - A}$ is sufficiently small
    for the objective function to be strongly convex with some probability.
    This will only require $O(n \log n)$ steps of the angular phase
    of the algorithm per iteration of Alecton, as stated in the main
    body of the paper.  We will need $p$ iterations of the algorithm to
    recover a rank-$p$ estimate, so a total $O(n p \log n)$ iterations
    will be required.
  \item Use a descent method, such as coordinate descent, to recover
    additional precision of the estimate.  This method is necessarily
    more heavyweight than an SGD scheme (see Section \ref{ssLowerBoundSGD}
    for the reason why an SGD scheme cannot achieve a linear rate),
    but it will converge monotonically at a linear rate to the exact
    solution matrix $A$.
\end{itemize}
This \emph{hybrid method} is in some sense a best-of-both worlds approach.
We use fast SGD steps when we can afford to, and then switch to slower
coordinate descent steps when we need additional precision.


\bibliographystylesec{plainnat} 
\bibliographysec{references}


\end{document}